\documentclass[letterpaper, 10pt, conference]{ieeeconf}  

\IEEEoverridecommandlockouts                              

\usepackage{graphics} 
\usepackage{multirow}
\usepackage{hyperref}
\usepackage{booktabs}
\usepackage{epsfig} 
\usepackage{mathptmx} 
\usepackage{times} 
\usepackage{amsmath} 
\usepackage{amssymb}  
\usepackage{color}
\usepackage{nicematrix}
\usepackage{subfig,graphicx}
\usepackage{url}
\usepackage{algorithm}
\usepackage{algpseudocode}
\usepackage{cite}
\newtheorem{theorem}{Theorem}
\newtheorem{definition}{Definition}

\newtheorem{lemma}{Lemma}
\newtheorem{remark}{Remark}

\newtheorem{assumption}{Assumption}
\newtheorem{example}{Example}

\graphicspath{{./Images/}}
\begin{document}
	
	\title{\LARGE \bf
		Efficient Planar Pose Estimation via UWB Measurements}
	\author{Haodong Jiang$^{1,2}$, Wentao Wang$^{3}$, Yuan Shen$^{4}$, Xinghan Li$^{3}$, Xiaoqiang Ren$^{5}$, Biqiang Mu$^{6}$, and Junfeng Wu$^{2,1}$
		\thanks{$^1$: Shenzhen Institute of Artificial Intelligence and Robotics for Society~(AIRS), {\tt\small haodongjiang@link.cuhk.edu.cn}. $^2$: School of Data Science, the Chinese University of HongKong, Shenzhen, P. R. China, {\tt\small junfengwu@cuhk.edu.cn}. $^3$: College of Control Science and Engineering and the State Key Laboratory of Industrial Control Technology, Zhejiang University, Hangzhou, P. R. China, {\tt\small \{wwt1999,11832014\}}@zju.edu.cn. $^4$: School of Electronic and Optical Engineering, Nanjing University of Science and Technology, Nanjing, P. R. China, {\tt\small syuan@njust.edu.cn}. $^5$: School of Mechatronic Engineering and Automation, Shanghai University, Shanghai, P. R. China {\tt\small xqren@shu.edu.cn}. $^6$: Key Laboratory of Systems and Control, Institute of Systems Science, Academy of Mathematics and Systems Science, Chinese Academy of Sciences, Beijing, P. R. China, {\tt\small bqmu@amss.ac.cn}}%
	}
	\maketitle
	\thispagestyle{empty}
	\pagestyle{empty}

	\begin{abstract}
		State estimation is an essential part of autonomous systems. Integrating the Ultra-Wideband~(UWB) technique has been shown to correct the long-term estimation drift and bypass the complexity of loop closure detection. However, few works on robotics treat UWB as a stand-alone state estimation solution. The primary purpose of this work is to investigate planar pose estimation using only UWB range measurements. We prove the excellent property of a two-step scheme, which says we can refine a consistent estimator to be asymptotically efficient by one step of Gauss-Newton iteration. Grounded on this result, we design the \textbf{GN-ULS} estimator, which reduces the computation time significantly compared to previous methods and presents the possibility of using only UWB for real-time state estimation.
	\end{abstract}
	\section{INTRODUCTION}
	\subsection{Background}
	State estimation is a fundamental prerequisite for an intelligent mobile robot to realize tasks such as obstacle avoidance and path planning. In recent years, significant efforts have been devoted to achieving high-performance and real-time state estimation using onboard sensors such as IMU, cameras, and lidars. However, these methods confront issues such as long-term drift~\cite{zhang2017low} and low robustness in geometrically degenerated environments~\cite{zhen2019estimating}. To overcome the above-mentioned challenges, we can integrate external information such as GPS in state estimation~\cite{shan2020lio}.
	
	Ultra-Wideband~(UWB) is a radio technology that is robust to multi-path effect and can provide precise TOA or TDOA measurements~\cite{schmid2019accuracy}. UWB is traditionally used for localization~\cite{gonzalez2009mobile, fang2020graph,zeng2022localizability,xue2021improving,zeng2023consistent}, while many recent works~\cite{song2019uwb, nguyen2021viral, wang2017ultra, cao2021vir} integrate UWB to realize drift-free state estimation in GPS-denied environments, few works on robotics investigate using UWB independently for pose estimation.
	
	This work considers estimating a robot's pose via only UWB range measurements obtained with a symmetric two-way TOA measuring technique. We focus on the planar case as shown in Fig.~\ref{fig::cover_figure}, which has many critical applications such as search and rescue robots~\cite{Queralta_Rescue_robots} and indoor service robots~\cite{Chung_indoor_service_robots}. We recognize it as the \textbf{Rigid Body Localization~(RBL)} problem in the signal processing community. 
	
	Through the literature review, we find that the pose estimators' statistical efficiency has not been well studied, and the estimators' computational complexity requires further reduction to realize real-time estimation. In this work, we adopt a two-step scheme and develop a closed-form estimator, which is asymptotically efficient under mild conditions related to anchor geometry, and significantly reduces computation time. We also conduct simulations and experiments to demonstrate our method's statistical and computational efficiency.

	\begin{figure}[t]
		\centering
		\includegraphics[width=0.45\textwidth]{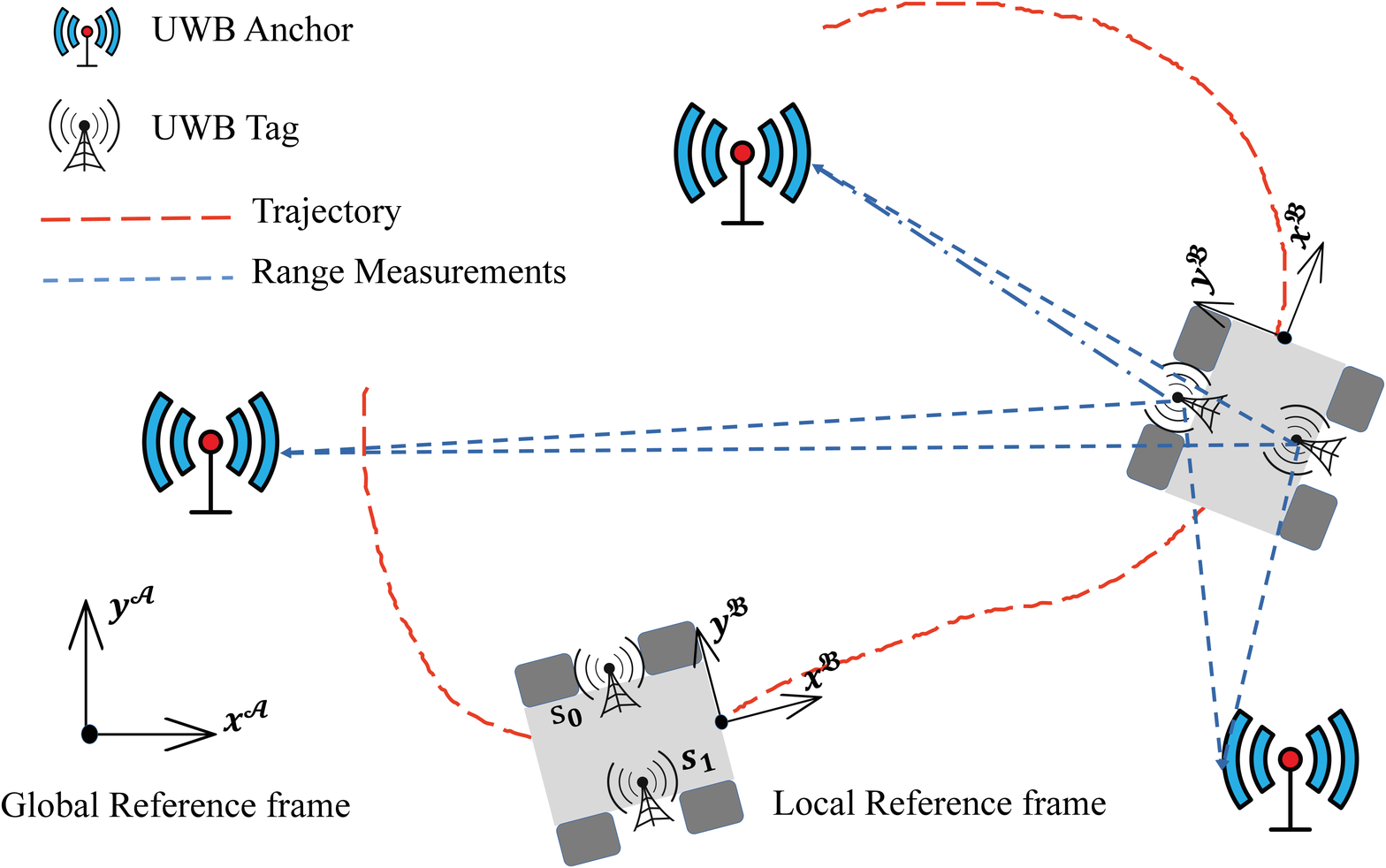}
		\caption{Planar Pose Estimation via UWB Measurements}
		\label{fig::cover_figure}
		\vspace{-20pt}
	\end{figure}
	
	\subsection{Related Work on Planar Rigid Body Localization}
	The maximum likelihood~(\textbf{ML}) formulation of RBL under i.i.d Gaussian noise assumption is a constrained weighted least squares~(LS) problem~\eqref{ML}. However, the ML estimate is difficult to obtain due to the nonconvexity and nonlinearity of~\eqref{ML}. A common practice is to apply the least squares methodology to the squared range measurements and formulate the squared least squares~(\textbf{SLS}) problem~\eqref{SLS}. 
	
	As far as we know, the work~\cite{chepuri2014rigid} is the first to formulate the RBL problem, which proposes to modify the \textbf{SLS} problem by projecting the squared measurements onto the null space of unit vectors. This operation eliminates the quadratic term and makes the problem linear, and the idea is followed by the work \cite{chen2015accurate} and our work. We term the resulting formulation~\eqref{PSLS} projected squared least squares~(\textbf{PSLS}). Based on \textbf{PSLS}, the work~\cite{chepuri2014rigid} solves a weighted orthogonal Procrustes problem by Gauss-Newton algorithms and obtains the initial value from simpler problems with closed-form solutions. The work~\cite{chepuri2014rigid} also derives the unitarily constrained Cramer-Rao lower bound~(CRLB). The work \cite{chen2015accurate} harnesses the structure of the rotation matrix and formulates a generalized trust region subproblem~(GTRS), and the solution is refined on the linearized \textbf{SLS} problem. The work~\cite{jiang2019sensor} uses semidefinite relaxation and formulate \textbf{SLS} as a semidefinite program~(SDP). The solution is refined by one step of Gauss-Newton iteration on the \textbf{ML} problem. 
	
	To sum up, the ML estimate for the planar RBL problem is difficult to obtain. Previous works turn to the \textbf{SLS} problem and use different techniques to make the problem linear. However, earlier studies do not rigorously evaluate the proposed estimators' deviation from the ML estimator or their statistical efficiency. Considering the growing interest in provably optimal state inference~\cite{rosen2021advances}, we believe these topics deserve careful study. The theoretical results also motivate the design of a faster optimal estimator.
	\subsection{Contributions}
	We summarize our contributions as follows:
	\begin{itemize}
		\item[(i)] We design a closed-form planar pose estimator using only UWB range measurements, which has $O(n)$ computational complexity and provably converges to the ML estimator as the measurement number $n$ increases. Given the high sampling rate of practical UWB systems~\cite{grossiwindhager2019snaploc}, our work can play a valuable role in applications.
		\item[(ii)] We propose mild conditions related to anchor geometry, under which the ML estimator, therefore our method is asymptotically efficient, i.e., the estimate can converge to the true pose with minimum variance.
		\item[(iii)] We conduct experiments in an indoor environment and elaborate on the data preprocessing procedure. The dataset and code are available on our website\footnote{https://github.com/SLAMLab-CUHKSZ/Efficient-Pose-Estimation-via-UWB-measurements}.
	\end{itemize}
	
	We organize the paper as follows. Section~\ref{section::Planar Rigid Body} gives formulations of the planar RBL problem. Section~\ref{Section:: Consistent Estimator} develops the \textbf{GN-ULS} estimator in two steps. Section~\ref{section::Asymptotic ML} gives conditions under which the ML estimator, therefore the \textbf{GN-ULS} estimator, is asymptotically efficient.  Section~\ref{section:: simulations} compare different estimators via various simulations. Section~\ref{section:: simulations} introduces data collection and discusses experiments on static and dynamic datasets. Section~\ref{section::conclusion} concludes the paper and discusses future works.

	\textbf{Notations:} All vectors are column vectors and denoted by bold, lower case letters; matrices are denoted by bold, upper case letters; reference frames are denoted with $\mathcal {A}, \mathcal {B}, \mathcal {C}$. ${\rm vec}({\bf A})$ is a column vector obtained by stacking the columns of matrix ${\bf A}$. ${\rm Null}({\bf A})$ is the null space of ${\bf A}$. We use ${\bf I,1,0}$ to denote identity matrices, unit vectors or matrices, and zero vectors or matrices. The symbol $\odot$ represents the Hadamard product and $\otimes$ the Kronecker product. For a list of vectors, we use the tuple notation $(\bf{v_1},\bf{v_2}\dots,\bf{v_n})$ for 
	$[\bf{v_1^{\top},v_2^{\top}\dots,v_n^{\top}}]^{\top}$. We use the notation $\textbf{a}^o$ for the true value of an unknown variable $\textbf{a}$. We use the notation  $O_p(\mathbf{1})$ for vectors or matrices whose entries are $O_p(1)$, i.e., stochastically bounded. We use the notation $o_p(\mathbf{1})$ for vectors or matrices whose entries are $o_p(1)$, i.e., convergent in probability towards zero.
	
		\section{Planar Rigid Body Localization}\label{section::Planar Rigid Body}
	\subsection{Range Measurement Model}
	Let $\mathcal {A}$ be the global frame, and let $\mathcal{B}$ be the local frame of the rigid body. The z-axes of $\mathcal{A}$ and $\mathcal{B}$ are aligned.
	
	Let $\textbf{a}_m^{\mathcal{A}}\in\mathbb{R}^2, m\in\{1,\ldots ,M\}$ be the planar coordinates of $M$ anchors in $\mathcal{A}$. Let $\textbf{s}_i^{\mathcal {B}}\in\mathbb{R}^2, i\in\{1,\ldots ,N\}$ be the planar coordinates in $\mathcal{B}$ of $N$ tags fixed on the rigid boy. The planar pose of $\mathcal{B}$ with respect to $\mathcal{A}$ can be represented by a rotation matrix ${\mathbf{R}}^o\in{\rm SO}(2)$ and a position vector ${\mathbf{t}^o}\in\mathbb{R}^2$. Each anchor can range with each tag repeatedly. For the brevity of formulation, we consider repeatedly ranging for $T$ times equivalent to deploying $T$ anchors on the same site. In total, we have $M_{T}\triangleq{MT}$ anchors and $n=NM_{T}$ measurements. 
	
	We denote by $d_{im}$ the distance measurement betweem the $m$-th anchor and the $i$-th tag, denote by $\Delta h_{im}$ the height difference between the $m$-th anchor and the $i$-th tag. The distance measurement model is as follows:
	\begin{equation}\label{measurement model}
		d_{im}=\sqrt{\|\textbf{a}_m^{\mathcal {A}}-\textbf{R}^o\textbf{s}_i^{\mathcal {B}}-\textbf{t}^o\|^2+(\Delta h_{im})^2}+r_{im}
	\end{equation}
	where $r_{im}$ is the additive measurement noise.  
	
	The RBL problem is formulated as follows: given $\textbf{a}_m^{\mathcal{A}}$, $\textbf{s}_i^{\mathcal {B}}$, $\Delta h_{im}$ and the distance measurements $d_{im}$, estimate the pose of the rigid body, i.e., the rotation $\textbf{R}^o$ and the position $\textbf{t}^o$.
	
	\subsection{Assumptions}
	We make the following assumptions:
	
	\begin{assumption}\label{assumption::same height}
		Anchors and tags are on the same height.
	\end{assumption}
	\begin{remark}
		We assume $\Delta h_{im}=0$ for simplicity, modification for $\Delta h_{im}\neq 0$ is detailed in Appendix~\ref{Appendix::Extension}.
	\end{remark}
	
	\begin{assumption}\label{assumption::noise}
		$r_{im}$'s are i.i.d Gaussian noises with zero mean, finite and known standard deviation $\sigma_{im}>0$.
	\end{assumption}
	\begin{assumption}\label{assumption::anchor distribution}
		The sample distribution $F_m$ of the sequence  $\bf{a_1^{\mathcal{A}}},\bf{a_2^{\mathcal{A}}},\dots$ converges to some distribution $F_\mu$.
	\end{assumption}
	\begin{example} \label{example_random_vector}
		Suppose ${\bf a}_{m}^{\mathcal{A}}$'s are independent samples from distribution function $F_{\mu}$, $F_m$ converges to $F_{\mu}$ as $M$ increases.
	\end{example}
	\begin{example} \label{example_fix_anchors}
		As the number $T$ of repeated ranging increases, $F_m$ converges to $F_{\mu}$. Denote the measure induced by $F_{\mu}$ as $\mu$, we have $\mu({\bf a}_{m}^{\mathcal{A}})=\frac{1}{M}$ for $m=1,\dots,M$.
	\end{example}
	\begin{assumption}\label{assumption::unique localizability}
		At least three non-colinear anchors exist, and at least two tags non-colinear with the origin $O^{\mathcal{B}}$ of the local reference frame.
	\end{assumption}
	\begin{remark}
		The necessary and sufficient condition for the planar pose to be observable requires at least three non-colinear anchors and at least two tags. A simple proof is given in Appendix~\ref{Appendix:: observability}. Under specific constraints, the required sensors can be further reduced, such as in work~\cite{lee2021ultrawideband}. We make a slightly stronger assumption in this paper to ease the proof of Lemma~\ref{lamma:rank} in Appendix~\ref{Appendix::lemma1}.
	\end{remark}
	\begin{assumption} \label{assumption::deployment_of_anchor}
		Consider the limit measure $\mu$ in Assumption \ref{assumption::anchor distribution}, there does not exist any line $\mathcal L$ such that $\mu(\mathcal L)=1$.
	\end{assumption}
	
	\subsection{Maximum Likelihood Formulation}
	The ML estimates $\hat{{\bf R}}_{\rm ML},\hat{{\bf t}}_{\rm ML}$ of the planar RBL problem are the solution to the following optimization problem:
	\begin{subequations}\label{ML}
		\begin{align}
			\hbox{\textbf{(ML)}}~\mathop{\rm min~}\limits_{\textbf{R,t}} ~&\sum_{i=1}^{N}\sum_{m=1}^{M_T}\frac{(d_{im}-\|{\bf a}_{m}^{\mathcal{A}}-
				{\bf Rs}_i^{\mathcal{B}}-{\bf t}\|)^2}{\sigma_{im}^2} \\
			\mathop{\rm s.t.} ~~& \textbf{R}\in {\rm SO}(2), ~\textbf{t} \in \mathbb{R}^2,
		\end{align}
	\end{subequations}
	\subsection{Squared Least Squares Formulation}
	The \textbf{ML} formulation is non-linear, non-convex, and difficult to solve. We square both sides of~\eqref{measurement model} and have:
	\begin{subequations}\label{squared model}
		\begin{align}
			d_{im}^2 &= (\|\textbf{a}_m^{\mathcal {A}}-\textbf{Rs}_i^{\mathcal {B}}-\textbf{t}\|+r_{im})^2 \label{smodel.1}\\
			&=\|\textbf{a}_m^{\mathcal {A}}\|^2-2{\textbf{a}_m^{\mathcal {A}}}^{\top}(\textbf{Rs}_i^{\mathcal {B}}+\textbf{t})+\|\textbf{Rs}_i^{\mathcal {B}}+\textbf{t}\|^2\nonumber\\&+2\|\textbf{a}_m^{\mathcal {A}}-\textbf{Rs}_i^{\mathcal {B}}-\textbf{t}\|r_{im}+r_{im}^2\label{smodel.3}
		\end{align}
	\end{subequations}
	
	We subtract $\sigma_{im}^2$ from both sides of \eqref{smodel.3} and have:
	\begin{equation}\label{smodel.4}
		d_{im}^2-\sigma_{im}^2=\|\textbf{a}_m^{\mathcal {A}}\|^2-2{\textbf{a}_m^{\mathcal {A}}}^{\top}(\textbf{Rs}_i^{\mathcal {B}}+\textbf{t})+\|\textbf{Rs}_i^{\mathcal {B}}+\textbf{t}\|^2+e_{im},
	\end{equation}
	where $e_{im}=2\|\textbf{a}_m^{\mathcal {A}}-\textbf{Rs}_i^{\mathcal {B}}-\textbf{t}\|r_{im}+(r_{im}^2-\sigma_{im}^2)$ has zero mean. 
	
	Stacking \eqref{smodel.4} over all measurements for $i$-th tag gives
	\begin{equation}\label{vector form}
		{\bf d}_i=-2{\bf A}^{\top}({\bf Rs}_i^{\mathcal {B}}+{\bf t})+\|{\bf Rs}_i^{\mathcal {B}}+\textbf{t}\|^2 {\bf 1}_{M_{T}}+\textbf{e}_i
	\end{equation}
	where
	$$\textbf{A}=\left[ \textbf{a}_1^{\mathcal {A}} ~ \textbf{a}_2^{\mathcal {A}}  \cdots  \textbf{a}_{M_T}^{\mathcal {A}} \right]\in \mathbb{R}^{2\times M_{T}}~~\textbf{S}=\left[
	\textbf{s}_1^{\mathcal {B}}~  \textbf{s}_2^{\mathcal {B}} \cdots  \textbf{s}_N^{\mathcal {B}} \\
	\right]\in \mathbb{R}^{2\times N}
	$$
	$$ \bf{d}_i=\begin{bmatrix}
		d_{i1}^2-\|\textbf{a}_1^{\mathcal {A}}\|^2-\sigma_{i1}^2 \\
		\vdots \\
		d_{iM_{T}}^2-\|\textbf{a}_{M_{T}}^{\mathcal {A}}\|^2-\sigma_{iM_{T}}^2\\
	\end{bmatrix}
	~~~{\rm and}~~~\bf{e}_i=\begin{bmatrix}
		e_{i1} \\
		\vdots \\
		e_{iM_{T}}\\
	\end{bmatrix}.
	$$
	The squared least squares problem is formulated as follows:
	\begin{subequations}\label{SLS}
		\vspace{-10pt}
		\begin{align}
			\hbox{\textbf{(SLS)}}~\mathop{\rm min~}\limits_{\textbf{R,t}} ~&\sum_{i=1}^{N}\|{\bf d}_i+2{\bf A}^{\top}({\bf Rs}_i^{\mathcal {B}}+{\bf t})-\|{\bf Rs}_i^{\mathcal {B}}+{\bf t}\|^2 {\bf 1}_{M_T}\|^2_{ \Sigma_{{\bf e}_i}} \\
			\mathop{\rm s.t.} ~~& {\bf R}\in {\rm SO}(2), ~{\bf t} \in \mathbb{R}^2,
		\end{align}
	\end{subequations}
	where $\|\cdot\|^2_\Sigma\triangleq{(\cdot)^{\top}\Sigma^{-1}(\cdot)}$, and ${\Sigma_{{\bf e}_i}}$ is the covariance of ${\bf e}_i$.
	
	Previous works differ in the way they deal with the quadratic term $\|\textbf{Rs}_i^{\mathcal {B}}+\textbf{t}\|^2$. The works \cite{chepuri2014rigid,chen2015accurate} multiply both sides of~\eqref{vector form} by projection matrix or orthonormal basis of ${\rm Null}({\bf 1})$ and formulate a linear LS problem, while \cite{jiang2019sensor} expands the quadratic term and formulate an SDP problem.

	\section{An Efficient Planar Pose Estimator}\label{Section:: Consistent Estimator}
	We design the planar pose estimator in two steps, grounded on the following theorem:
	\begin{theorem}\label{theorem: little op}
		Suppose that $\hat{\bf R}$ and 
		$\hat {\bf t}$ are $\sqrt{n}$-consistent estimates of
		$\bf R^o$ and $\bf t^o$. The estimates $\hat{{\bf R}}_{\rm GN},\hat{{\bf t}}_{\rm GN}$ obtained by one step of Gauss-Newton iteration on the ML problem~\eqref{ML} converge in probability to the ML estimates as measurement number $n$ increases, and
		$$
		\hat{{\bf R}}_{\rm GN}-
		\hat{{\bf R}}_{\rm ML}=o_p(\mathbf{1}/\sqrt{n}),~~\hat{{\bf t}}_{\rm GN}-\hat{{\bf t}}_{\rm ML}=o_p(\mathbf{1}/\sqrt{n}),
		$$
		The proof is given in the Appendix~\ref{Appendix:: little op}.
	\end{theorem}
	\begin{remark}
		Theorem~\ref{theorem: little op} motivates us to design an as computationally efficient as possible $\sqrt{n}-$consistent estimator and refine it by one step of Gauss-Newton iteration. We discuss these two steps in section~\ref{section::ULS} and ~\ref{Section::one step GN} respectively. 
	\end{remark}
	
	\subsection{Unconstrained Least Squares Estimator}\label{section::ULS}
	Denote the projection matrix onto ${\rm Null}({\bf 1_{M_T}})$ as ${\bf P}={\bf I}_{M_{T}}-({\bf 1}_{M_{T}}{\bf 1}_{M_{T}}^{\top})/M_{T}$. Multiply both sides of~\eqref{vector form} by ${\bf P}$ and eliminate the quadratic term, we formulate the following problem~\cite{chepuri2014rigid}:
	\begin{subequations}\label{PSLS}
		\vspace{-5pt}
		\begin{align}
			\hbox{\textbf{(PSLS)}}~\mathop{\rm min~}\limits_{{\bf{R,t,n}}} ~&\sum_{i=1}^{N} \|{\bf Pd}_i+2{\bf P}{\bf A}^{\top}({\bf Rs}_i^{\mathcal {B}}+{\bf t})\|^{2}_{\Sigma_{{\bf \bar e}_i}}\label{LS-2_objective}\\
			\mathop{\rm s.t.} ~~& \textbf{R}\in {\rm SO}(2), ~\textbf{t} \in \mathbb{R}^2\label{LS-2 constraint-1}
		\end{align}
	\end{subequations}
	where $\Sigma_{{\bf \bar e}_i}$ is the covariance matrix for ${\bf Pe}_i$. 
	
	The rotation matrix $\mathbf{R}$ has a nice structure, such that we can parameterize ${\bf R}$ by an angle $\theta\in[0,2\pi)$:~
	\begin{equation}\label{eqn::angle paramerization}
		{\bf R}(\theta)\triangleq \begin{bmatrix}
			\cos{\theta} & -\sin{\theta} \\
			\sin{\theta} & \cos{\theta} 
		\end{bmatrix} .
	\end{equation}
	Motivated by~\eqref{eqn::angle paramerization}, we can use a unit-length vector $\textbf{y}=(y_1,y_2)\in\mathbb{R}^2$ to parameterize $\mathbf{R}$, where $y_1,y_2$ correspond to $\sin(\theta),\cos(\theta)$ respectively. Denote ${\bf AP}$ as $\bar{\bf{A}}$, ${\bf Pd}_i$ as $\bar{\bf{d}_i}$, and ${\bf Pe}_i$ as $\bar{\bf{e_i}}$, we formulate the following GTRS problem:
	\begin{equation}\label{GTRS}
		\mathop{\rm min~}\limits_{\textbf{y,t}} \|{\bar{\bf{d}}}-\textbf{H}_1\Gamma\textbf{y}-\textbf{H}_2\textbf{t}\|^2_{\Sigma_{\bar e}^{-1}}\quad\text{s.t. }  \|\bf{y}\|^2=1
	\end{equation}
	where $\Gamma=\begin{bmatrix}
		0 & 1 & -1 & 0\\
		1 & 0 & 0 & 1
	\end{bmatrix}^{\top}$,
	${\bar{\bf{d}}}=(\bar{\bf{d}}_1,\dots,\bar{\bf{d}}_N)$,\\ ${\bf H}_1=-2\bf{S}^{\top}\otimes\bar{\bf{A}}^{\top}$, and ${\bf H}_2=-2\bf{1}_N\otimes\bar{\bf{A}}^{\top}.
	$
	
	The matrix $\Sigma_{{\bf \bar e}}$ is dense and dependent on true distance $d_{im}^{o}$'s~\cite{chepuri2014rigid}. We discard the constraint and covariance for efficiency and simplicity and solve the resultant LS problem.
	\begin{lemma}\label{lamma:rank}
		Under Assumption~\ref{assumption::unique localizability}, the design matrix ${\bf H}=[{\bf H}_1\Gamma, \bf{H}_2]$ is full column rank, and the unique solution to the resultant LS problem is given by:
		\begin{equation}\label{consitent estimator}
			\hbox{\textbf{(ULS)}}~~~~~~\left[
			\begin{array}{c}
				\mathbf{\hat{y}} \\
				\mathbf{\hat{t}} \\
			\end{array}
			\right]=(\textbf{H}^{\top}\bf{H})^{-1}\textbf{H}^{\top}{\bf{\bar{d}}},
		\end{equation}
	\end{lemma}
	\begin{theorem} \label{theorem::consistency_of_LS_lin_estimate}
		The \textbf{ULS} estimator is $\sqrt{n}$-consistent, i.e., {\rm
			$$  \left[
			\begin{array}{c}
				\mathbf{\hat{y}} \\
				\mathbf{\hat{t}} \\
			\end{array}
			\right]-  \left[
			\begin{array}{c}
				\mathbf{y}^{o} \\
				\mathbf{t}^{o} \\
			\end{array}
			\right]=O_p(\mathbf{1}/\sqrt{n}).$$}
		The proof is given in Appendix~\ref{Appendix::lemma1} and~\ref{Appendix:: big Op}.
	\end{theorem}
	
	Notice that the estimate ${\bf \hat{y}}$ from~\eqref{consitent estimator} is not constrained to have unit length, so we further project ${\bf \hat{R}}\triangleq\Gamma{
		\bf \hat{y}}$ onto ${\rm SO}(2)$. The matrix projection $\pi$ that maps an arbitrary matrix $\textbf{X}\in\mathbb{R}^{2\times 2}$ onto ${\rm SO}(2)$ is defined as 
	\begin{equation}\label{eqn:def_projection_map}
		\pi(\textbf{X})=\arg\min_{\textbf{W}\in {\rm SO}(2)}\|\textbf{X}-\textbf{W}\|_F^2.
	\end{equation}
	Let the SVD of $\textbf{X}$ be $\textbf{U}{\Sigma}{\bf V}^{\top}$, we have
	\begin{equation}
		\pi(\textbf{X})=\textbf{U}{\rm diag}([1, \det(\textbf{U}{\bf V}^{\top})]) {\bf V}^{\top}
	\end{equation}
	\begin{theorem} \label{theorem:: consistency after projection}
		\vspace{-10pt}
		The	projected estimate generated by~\eqref{eqn:def_projection_map} is $\sqrt{n}$-consistent, i.e., 
		$$
		\pi({\bf \hat{R}})-{\bf R}^{o}=O_p({\bf 1}/\sqrt{n}).
		$$
		The proof is given in Appendix~\ref{Appendix:: projection}.
		\begin{remark}
			By assigning $O^{\mathcal{B}}$ on tag $j$, the measurement model of tag $i$ becomes $d_{jm}=\|\textbf{a}_m^{\mathcal {A}}-\textbf{t}^o\|+r_{jm}$. We can thus estimate translation ${\mathbf t}^o$ first using $d_{jm}$'s, and then substitute $\hat{\mathbf t}$ to estimate the rotation matrix ${\mathbf R}^o$. Similarly, we can first localize the tags and then infer the pose, as in work~\cite{bonsignori2020estimation}. 
			
			The accuracy of the intermediate step is crucial for such divide-and-conquer schemes. Given a $\sqrt{M_T}-$consistent estimate for the intermediate step, the resultant pose estimator can achieve $\sqrt{n}-$consistency, as shown in section~\ref{section:: simulations} and proved in Appendix~\ref{Appendix::DAC}.
		\end{remark}
	\end{theorem}
	\subsection{One Step of the Gauss-Newton Iteration}\label{Section::one step GN}
	Given a $\sqrt{n}$-consistent estimate from the first step, we implement one step of Gauss-Newton iteration on the \textbf{ML} problem~\eqref{ML} and obtain the \textbf{GN-ULS} estimator. Using~\eqref{eqn::angle paramerization}, we can transform problem~\eqref{ML}~into an unconstrained one:
	\begin{equation}\label{ULS}
		\mathop{\rm min~}\limits_{\theta, {\bf t}} ~\sum_{i=1}^{N}\sum_{m=1}^{M_T}\frac{  (d_{im}-\|{\bf a}_{m}^{\mathcal{A}}-{\bf L}_i{\rm vec}\left({\bf R}(\hat{\theta}+\theta)\right)-{\bf t}\|)^2}{\sigma_{im}^2}
	\end{equation}
	where ${\bf L}_i=({\bf s}_i^{\mathcal{B}}\otimes {\bf I}_2)^{\top}\in\mathbb{R}^{2\times 4}$ and $\mathbf{R}(\hat{\theta})=\mathbf{\pi(\hat{R})}$. 
	
	We use $\theta=0, \mathbf{t}=\hat{{\bf t}}$ as the initial value. Write $f_{im}(\theta, {\bf t})\triangleq{\bf a}_{m}^{\mathcal{A}}-{\bf L}_i{\rm vec}\left({\bf R}(\hat{\theta}+\theta)\right)-{\bf t}$ and $f_{im}^{(0)}\triangleq f_{im}(0, \hat {\bf t})$. The derivatives of $\|f_{im}^{(0)}\|$ w.r.t $\theta$ and ${\bf t}$ are:
	\begin{equation*}
		\frac{\partial \|f_{im}^{(0)}\|}{\partial (\theta,{\bf t})}=\begin{bmatrix}
			-\frac{1}{\|f_{im}^{(0)}\|}
			\Psi^{\top} ({\bf I}_2\otimes \hat{\bf 
				R}^{\top})  {\bf L}_i^{\top}\,f_{im}^{(0)}\\
			-\frac{1}{\|f_{im}^{(0)}\|} f_{im}^{(0)}
		\end{bmatrix}
	\end{equation*}
	where $\Psi=\frac{\partial {\rm vec}({\bf R}(0))}{
		\partial \bf \theta}=\begin{bmatrix}
		0~1~-1~0
	\end{bmatrix}^{\top}$.
	Stacking the rows $\frac{\partial \| f_{im}(0,{\bf\hat t})\|}{\partial (\theta,{\bf t})^{\top}}$ gives the matrix ${\bf J}_0$. Stacking $\|f_{im}^{(0)}\|$ gives the vector ${\bf f}(0,\hat{\bf t})$. Denote the covariance matrix of $r_{im}$ as $\Sigma_{n}$, the one step of Gauss-Newton iteration $(\hat{\theta},\hat{{\bf t}})_{\rm GN}$ writes:
	\begin{equation}\label{eqn::one step GN}
		(\hat{\theta},\hat{{\bf t}})_{\rm GN}=
		(0,\hat{\bf t})+
		({\bf J}_0^{\top}\Sigma_{n}^{-1}{\bf J}_0)^{-1}{\bf J}_0^{\top}\Sigma_{n}^{-1}\big({\bf d}-{\bf f}(0,\hat{{\bf t}})\big)
	\end{equation}
	as such, we obtain the \textbf{GN-ULS} estimates, 
	\begin{equation}\label{GN-ULS}
		\hbox{\textbf{(GN-ULS)}}~~~~~~\hat{{\bf R}}_{\rm GN}={\bf R}(\hat{\theta}+\hat{\theta}_{\rm GN}),~~ {\bf\hat{t}}_{\rm GN}={\bf\hat{t}}_{\rm GN}
	\end{equation}
	We summarize the \textbf{GN-ULS} estimator in Algorithm~\ref{algorithm: GN-ULS}.
	\begin{algorithm}
		\caption{Unconstrained Least Squares estimator refined by one-step of Gauss-Newton iteration~(GN-ULS)}
		\label{algorithm: GN-ULS}
		\begin{algorithmic}[1]
			\Statex{\bf Input:} $d_{im}$, $\sigma_{im}$, ${\bf a}_m^{\mathcal {A}}$ and ${\bf s}_i^{\mathcal {B}}$.
			\Statex{\bf Output:} the estimates of ${\bf R^o}$ and ${\bf t}^o$.
			
			\State Construct ${\bf \bar{d}}\in\mathbb{R}^{n\times 1}$ and ${\bf H}\in\mathbb{R}^{n\times 4}$.
			
			\State Derive the \textbf{ULS} estimator as $(\textbf{H}^{\top}\bf{H})^{-1}\textbf{H}^{\top}{\bf{\bar{d}}}$.
			
			\State Project the \textbf{ULS} estimate onto $SO(2)$~\eqref{eqn:def_projection_map}.
			
			\State Construct ${\bf J}_0\in\mathbb{R}^{n\times 3}$ and ${\bf f}(0,{\bf\hat{t}})\in\mathbb{R}^{n\times 1}$.
			
			\State Implement one step of Gauss-Netwon iteration~\eqref{eqn::one step GN}.
			
			\State Obtain the \textbf{GN-ULS} estimator~\eqref{GN-ULS}.
		\end{algorithmic}
	\end{algorithm}
	
	\section{Statistical Efficiency Analysis}\label{section::Asymptotic ML}
	According to Theorem~\ref{theorem: little op}, the \textbf{GN-ULS} estimates converge to the ML estimates as $n$ increases. The following theorem describes the \textbf{GN-ULS} estimator's statistical efficiency.

	\begin{theorem}\label{theorem::ML}
		Under assumptions~\ref{assumption::anchor distribution}-\ref{assumption::deployment_of_anchor}, the \textbf{ML} estimates, therefore the \textbf{GN-ULS} estimates are asymptotically efficient, i.e., as measurement number $n$ increases, 
		\begin{align*}
			\text{CRLB}^{-1/2}\begin{bmatrix}
				{\rm vec}(\hat{{\bf R}}_{\rm ML}-{\bf R}^o)\\
				\hat{{\bf t}}_{\rm ML}-{\bf t}^o
			\end{bmatrix}&\xrightarrow{D} \mathcal{N}\left({\bf 0},{\bf I}\right),\\		\text{CRLB}^{-1/2}\begin{bmatrix}
				{\rm vec}(\hat{{\bf R}}_{\rm GN}-{\bf R}^o)\\
				\hat{{\bf t}}_{\rm GN}-{\bf t}^o
			\end{bmatrix}&\xrightarrow{D} \mathcal{N}\left({\bf 0},{\bf I}\right).
		\end{align*}
	The proof of Theorem~\ref{theorem::ML} is given in Appendix~\ref{Appendix:: theorem ML}. The formula for the CRLB and insights into sensor deployment from the view of CRLB are given in Appendix~\ref{Appendix::CRLB}.
	\end{theorem}
	\begin{remark}
		This theorem establishes the proposed estimator's optimality and legitimates us to approximate estimation covariance by CRLB, which is a vital problem in sensor fusion schemes. Although we need the true pose to calculate CRLB, we can reasonably use $\hat{\bf R}_{\rm GN},\hat{\bf t}_{\rm GN}$ as an alternative. 
	\end{remark}

	\section{Simulations and Discussions}\label{section:: simulations}
	We verify the asymptotic efficiency of the proposed \textbf{GN-ULS} estimator by comparing the root mean square error~(RMSE) with the lower bound, which is the square root of the trace of CRLB~\cite{chepuri2014rigid}, denoted as $\sqrt{CRLB}$. We compare with previous works~\cite{chen2015accurate} and \cite{jiang2019sensor}, denoted as the \textbf{GTRS} and the \textbf{GN-SDP} respectively. We also compare with the divide-and-conquer approach that localizes the tags first. For the implementation of this estimator, denoted as \textbf{DAC}, we follow work\cite{zeng2022global} for localization and solve a least squares problem for pose estimation. Given a $\sqrt{M_T}-$consistent position estimator, this \textbf{DAC} estimator can be proved to be $\sqrt{n}-$consistent, detailed in Appendix~\ref{Appendix::DAC}.
	\subsection{Simulation Setup}
	In our simulations, there are $M=3$ anchors deployed at $[50,0]^{\top}$, $[50,50]^{\top}$ and $[0,50]^{\top}$ in the global frame. There are 
	$N$ = 2 tags deployed at $[3,0]^{\top}$ and $[3,3]^{\top}$ in the local frame. The true pose is ${\bf t}^o =[0,25]^{\top}$ and ${\theta}^o =60^{\circ}$. 
	
	We run $L=1000$ Monte-Carlo experiments for each setting and report the average results. We use the chordal distance~\cite{hartley2013rotation} to calculate the RMSE for the rotation matrix:
	$$\text{RMSE}({\bf R})=\sqrt{
		\frac{1}{L}\sum_{l=1}^L \|{\bf \hat{R}}-{\bf R}^o\|_F^2
	}$$
	\subsection{Simulation Results}
	\subsubsection{Asymptotic efficiency under repeated ranging} We increase the number $T$ of repeated ranging. The noise standard deviation ${\sigma}_{mi}$'s are set to be $0.05[1, 2, 3, 4, 5, 6]$. As shown in Fig.~\ref{fig::trial1_a}, the \textbf{ULS} estimator and the \textbf{DAC} estimator are $\sqrt{n}-$consistent but not asymptotically efficient. The \textbf{GTRS} estimator deviates from the lower bound under large samples, and the main reason is that the estimate is refined on the \textbf{SLS} problem~\eqref{SLS} but not on the \textbf{ML} problem~\eqref{ML}. The \textbf{GN-ULS} and the \textbf{DAC} estimator are significantly more efficient with $O(n)$ complexity, as shown in Figure.~\ref{fig::trial1_b}. It also takes $O(n)$ computation to construct the optimization problems for the $\textbf{GN-ULS}$ and the $\textbf{GN-SDP}$ estimator, but solving the problems dominates the computation cost.
	
	\begin{figure} [tbp]
		\centering
		\vspace{2pt}
		\subfloat[\label{fig::trial1_a}]{
			\includegraphics[width=0.45\textwidth,height=3cm]{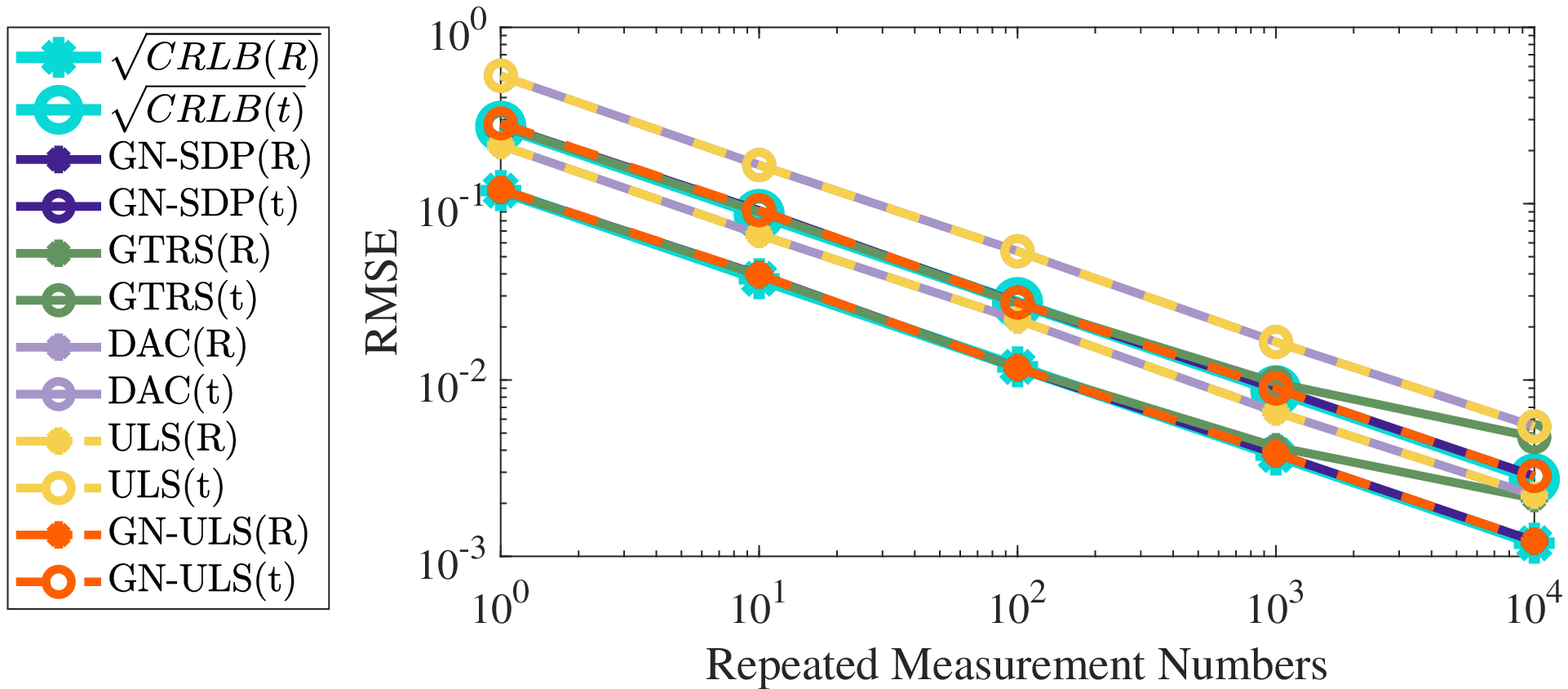}}
		\quad
		\subfloat[\label{fig::trial1_b}]{
			\includegraphics[width=0.45\textwidth,height=3cm]{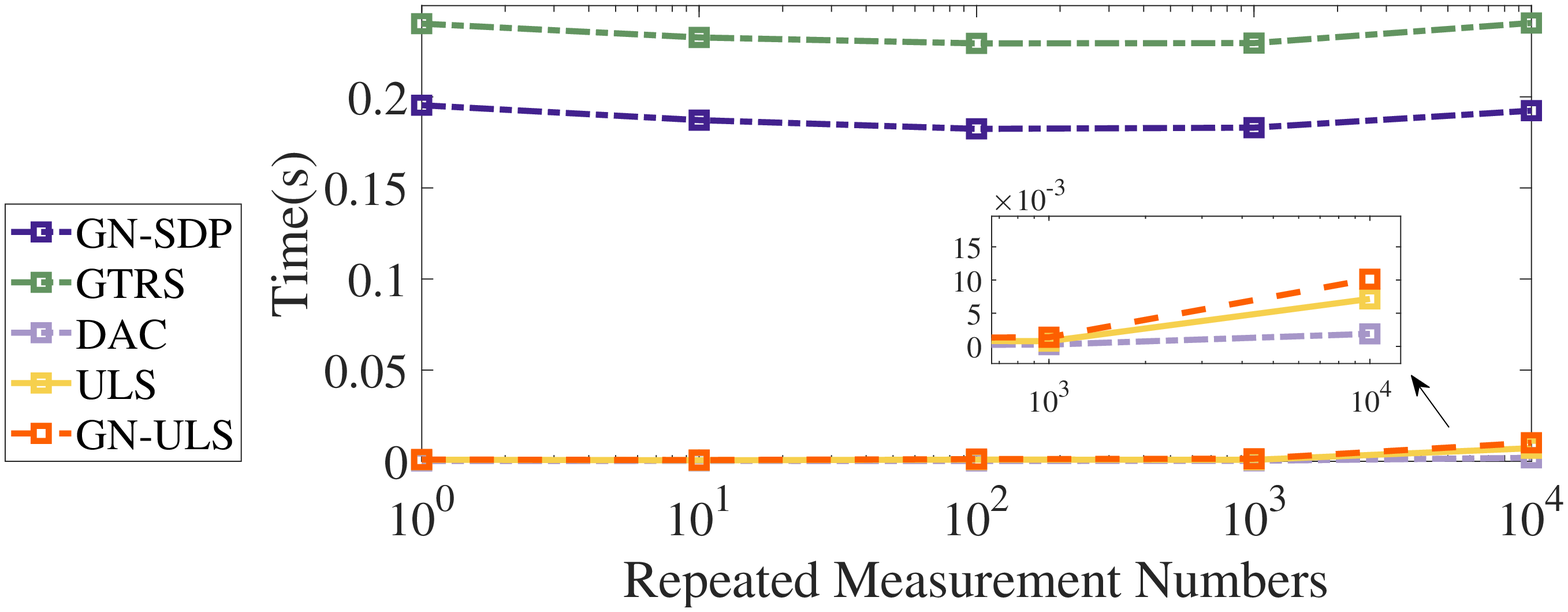}}
		\caption{Performance under repeated ranging. (a)~RMSE of different pose estimators and the lower bound as the number $T$ of repeated ranging increases. (b)~Average computation time of different pose estimators as $T$ increases.}
		\label{fig::trial1} 
		\vspace{-10pt}
	\end{figure}
	\subsubsection{Asymptotic efficiency under numerous anchors}
	We deploy new anchors uniformly on the simulation plane. The standard deviation $\sigma_{im}$ is set to be $0.1$. As shown in Fig.~\ref{fig::trial2}, the \textbf{GTRS} estimator deviates from the lower bound. When we deploy $10000$ anchors, the MATLAB CVX toolbox reports the SDP problem as infeasible. The same instability problem occurs in Fig.~\ref{fig::trial3} under very small noise.
	
	\begin{figure} [tbp]
		\vspace{2pt}
		\centering
		\subfloat[\label{fig::trial2}]{
			\includegraphics[width=0.45\textwidth,height=3cm]{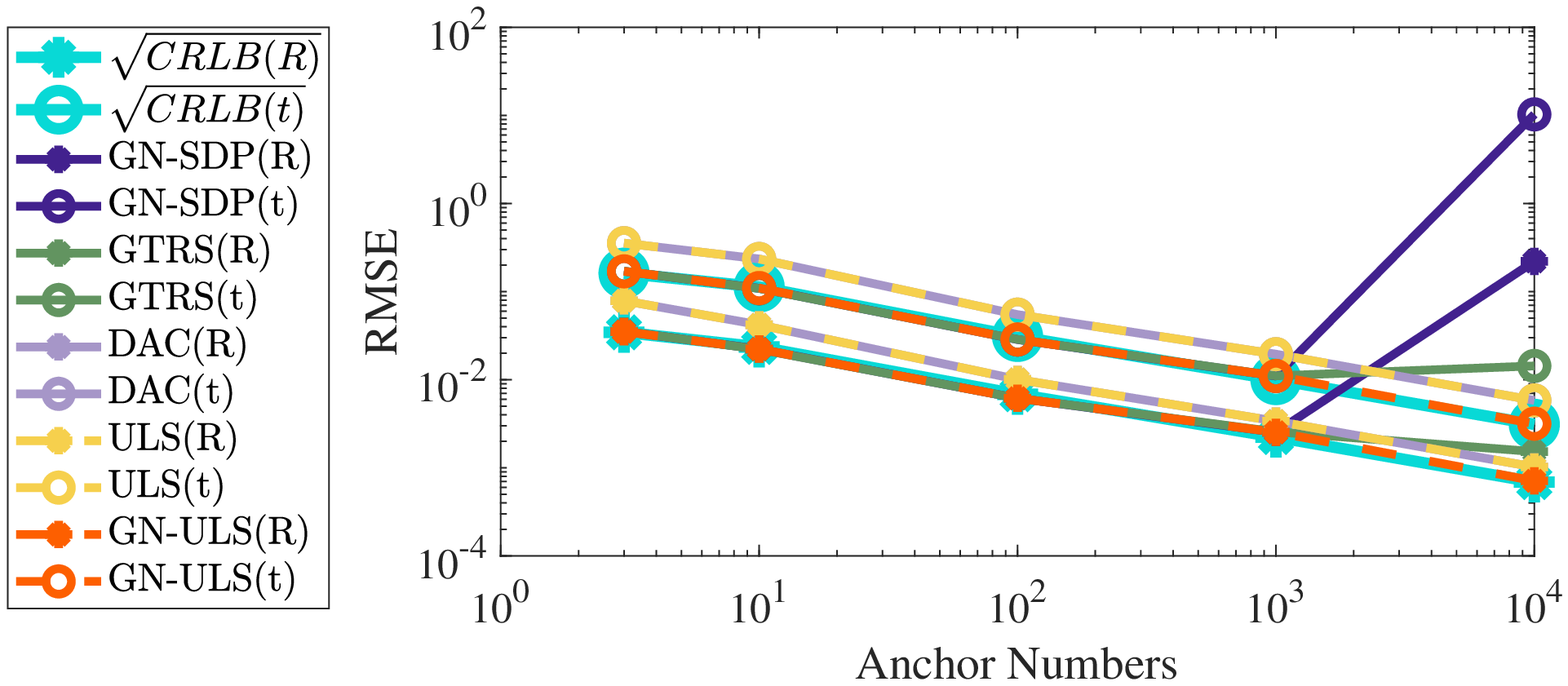}}
		\quad
		\subfloat[\label{fig::trial3}]{
			\includegraphics[width=0.45\textwidth,height=3cm]{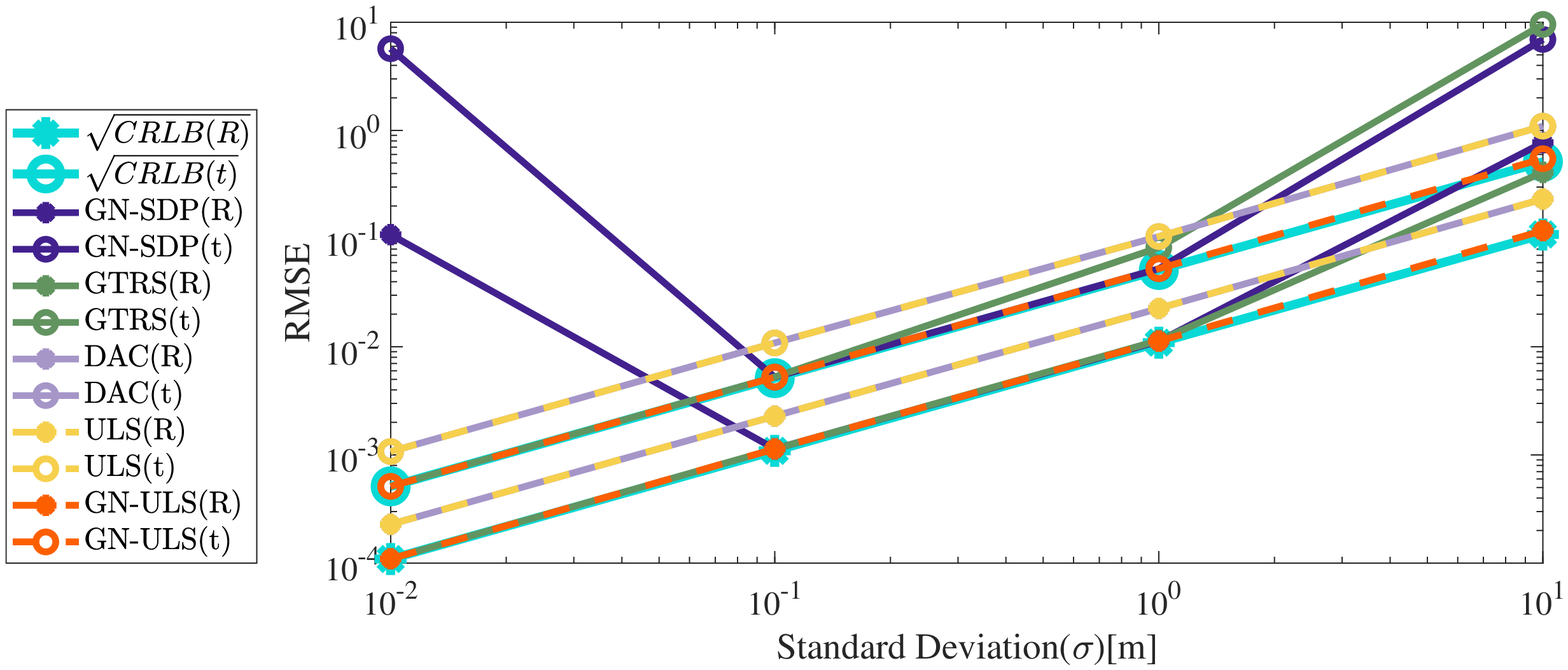}}
		\caption{Performance with numerous anchors and under different noise levels. (a)~RMSE of different pose estimators and the lower bound when the number $M$ of anchors increases. (b)~RMSE of different pose estimators and the lower bound as the standard deviation increases.}
		\label{fig::trial23} 
		\vspace{-10pt}
	\end{figure}
	\subsubsection{Asymptotic efficiency under large noise}
	We adjust the standard deviation $\sigma_{im}$ from $0.01$ to $10$, and set $T=1000$. As shown in Fig.~\ref{fig::trial3}, \textbf{GTRS} and \textbf{GN-SDP} deviate from the lower bound under large nose, this is because they omit $r_{im}^2$ in~\eqref{smodel.3} and use noisy measurements to calculate the covariance matrix $\Sigma_{{\bf \bar e}_i}$. 
	\subsection{Discussions}
	As the simulation results indicate, the \textbf{GN-ULS} attains theoretical lower bound, and performs better than computationally more complex estimators in stability and accuracy. As guaranteed by Theorem~\ref{theorem: little op}, the \textbf{DAC} estimator refined by one step of Gauss-Newton can perform comparably well to \textbf{GN-ULS}. But the \textbf{DAC} estimator appears more sensitive to outliers in dynamic experiments, as shown in Section~\ref{section::experiments}.
	
	\section{Experiments}\label{section::experiments}
	In this section, we introduce data collection and present experimental results on static and dynamic datasets.
	\subsection{System Overview}
	\begin{figure}[!bp]
		\centering
		\includegraphics[width=0.43\textwidth]{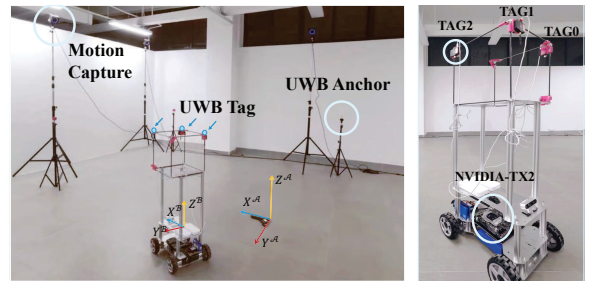}
		\caption{Experiment Setup}
		\label{fig::Experiment Setup}
		\vspace{-10pt}
	\end{figure}
	Fig.~\ref{fig::Experiment Setup} presents the experimental system and environment with an overall volume of $10m\times6m\times4m$. The system consists of a motion capture system~(OptiTrack: X22), UWB~(NoopLoop: LinkTrack), an Ackerman trolley platform, and an embedded processor~(NVIDIA: TX2). We use a carbon fibre cube frame, rather than directly using the trolley, as the rigid body to avoid blocking the UWB signal. Three UWB tags are fixed on the cube at approximately the same height as the eight anchors deployed in the environment.
	\subsection{Data Collection}
	We refer to the motion capture frame as the global frame. The UWB ranging frequency is 100Hz, and the motion capture system provides the ground truth of the pose at 120Hz, with $\pm0.15$ millimeters and $\pm0.5$ degrees claimed accuracy\footnote{https://optitrack.com/cameras/primex-22/}. During experiments, NVIDIA-TX2 unpacks the UWB data through its serial port and collects the motion capture system data through TCP. In dynamic datasets, we synchronize measurements using the system time of TX2 and perform interpolation to align the motion capture measurements with the UWB measurements\footnote{	We observe a positive bias between estimated yaw and the ground truth on all dynamic datasets and for all compared methods. We believe this phenomenon is due to imperfect synchronization caused by processing delay and unreliable ground truth for yaw when reflective markers are hidden from some cameras. As a remedy, we compensate all estimates by a negative degree on dynamic datasets.}
	
	\subsection{UWB Calibration}
	UWB ranging measurement is practically modeled as:
	\begin{equation}\label{model uwb}
		\hat{d}=d^o+f(d^o)+e,
	\end{equation}
	where $f(d^o)$ is a distance-related bias, and $e$ is a zero-mean Gaussian noise. We quiescent the trolley for a period of time and use the sample variance to estimate the standard deviation of $e$. The more demanding task is calibrating $f(d^o)$, which we assume to be a linear function of $d^o$~\cite{bellusci2008model}. We control the trolley to move around in the environment, collect calibration datasets and use the least squares method to fit $f(d^o)$. Considering the complex communication environment indoors, we implement outlier rejection before calibration. Similar to the methods in~\cite{cao2021vir,fang2020graph}, a range measurement at instant $t$ is rejected as an outlier if 
	\begin{equation}\label{eqn::rejection}
		d_t>\min{\{d_{t-k},\dots,d_{t-1}\}}+\frac{kv_{\rm max}}{f}+0.1,
	\end{equation}
	where $k$ is the length of the time window, $v_{\rm max}$ is the velocity upper bound during the experiment, $f$ is the ranging frequency, and $0.1m$ is a general error bound of UWB.

	\subsection{Static Datasets and Pose Estimation Results}
	We place the trolley at 7 different sites and change the orientation from $0^{\circ}$ to $300^{\circ}$ at an interval of approximately $60^{\circ}$. In total, we collect 42 static datasets, each lasting for around 100 seconds. We calculate the RMSE on all datasets and compare the average result. We choose two tags and three anchors and use the centimeter and the degree as units for Fig.~\ref{fig::static}. The result shows that all methods achieve similar accuracy to the ground truth system.
	\begin{figure}
		\vspace{5pt}
		\centering
		\includegraphics[width=0.43\textwidth,height=5cm]{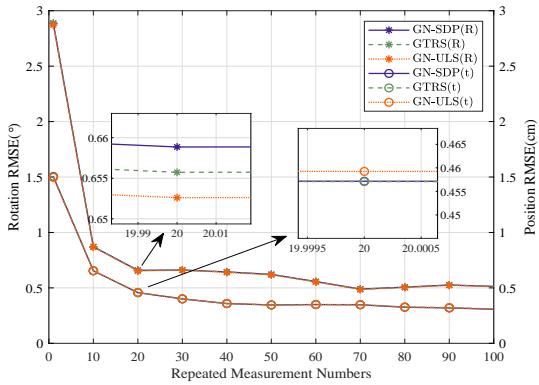}
		\caption{Static experiment under repeated ranging}
		\label{fig::static}
		\vspace{-15pt}
	\end{figure}
	\subsection{Dynamic Datasets and Pose Estimation Results}
	We control the trolley to move at different speeds~(0.49m/s, 0.23m/s, and 0.11m/s), and collect fast, medium, and slow datasets. We implement interpolation when an outlier occurs to provide real-time estimates. It is not reasonable to use repeated measurements in the dynamic case. Instead, we adopt all three tags and eight anchors.
	
	Fig.~\ref{fig::dynamic} visualizes the results on the fast dataset, where \textbf{GN-ULS} achieves an average RMSE of $3.97$ degrees and $3.01$ centimeters using 24 range measurements between eight anchors and three tags. Overall results on the three datasets are summarized in Table~\ref{table::dynamic experiment}, where $\textbf{GN-DAC}$ represents one step of Gauss-Newton iteration based on the $\textbf{DAC}$ estimator.

	\begin{table}[h]
		\centering
		\caption{Position and Rotation RMSE on dynamic datasets}
		\begin{tabular}{l|cllcllc|ccc}
			\toprule
			\hline
			\multicolumn{1}{c|}{\multirow{3}{*}{\textbf{Method}}} &
			\multicolumn{7}{c|}{\textbf{Position RMSE {[}cm{]}}} &
			\multicolumn{3}{c}{\textbf{Rotation RMSE {[}deg{]}}} \\ \cline{2-11} 
			\multicolumn{1}{c|}{} &
			\multicolumn{3}{c|}{\multirow{2}{*}{Fast}} &
			\multicolumn{3}{c|}{\multirow{2}{*}{Mid}} &
			\multirow{2}{*}{Slow} &
			\multicolumn{1}{c|}{\multirow{2}{*}{Fast}} &
			\multicolumn{1}{c|}{\multirow{2}{*}{Mid}} &
			\multirow{2}{*}{Slow} \\
			\multicolumn{1}{c|}{} &
			\multicolumn{3}{c|}{} &
			\multicolumn{3}{c|}{} &
			&
			\multicolumn{1}{c|}{} &
			\multicolumn{1}{c|}{} &
			\\ \hline
			ULS &
			\multicolumn{3}{c|}{3.73} &
			\multicolumn{3}{c|}{3.55} &
			3.48 &
			\multicolumn{1}{c|}{5.68} &
			\multicolumn{1}{c|}{4.74} &
			5.53 \\
			DAC & 
			\multicolumn{3}{c|}{3.60} &
			\multicolumn{3}{c|}{3.07} &
			3.14&
			\multicolumn{1}{c|}{11.55} &
			\multicolumn{1}{c|}{11.74} &
			11.18\\
			GN-ULS &
			\multicolumn{3}{c|}{3.01} &
			\multicolumn{3}{c|}{3.00} &
			2.59 &
			\multicolumn{1}{c|}{\textbf{3.97}} &
			\multicolumn{1}{c|}{3.40} &
			\textbf{3.75} 			\\
			GN-DAC &
			\multicolumn{3}{c|}{3.07} &
			\multicolumn{3}{c|}{3.07} &
			2.65&
			\multicolumn{1}{c|}{4.39} &
			\multicolumn{1}{c|}{3.80} &
			4.21\\
			GN-SDP &
			\multicolumn{3}{c|}{\textbf{3.00}} &
			\multicolumn{3}{c|}{\textbf{3.00}} &
			\textbf{2.58} &
			\multicolumn{1}{c|}{3.99} &
			\multicolumn{1}{c|}{\textbf{3.37}} &
			3.80 \\
			GTRS &
			\multicolumn{3}{c|}{3.22} &
			\multicolumn{3}{c|}{3.27} &
			2.83 &
			\multicolumn{1}{c|}{4.42} &
			\multicolumn{1}{c|}{3.97} &
			4.13\\ \hline
			\bottomrule
		\end{tabular}
		\label{table::dynamic experiment}
	\end{table}
	
	\begin{figure}
		\vspace{5pt}
		\centering
		\includegraphics[width=0.43\textwidth,height=5cm]{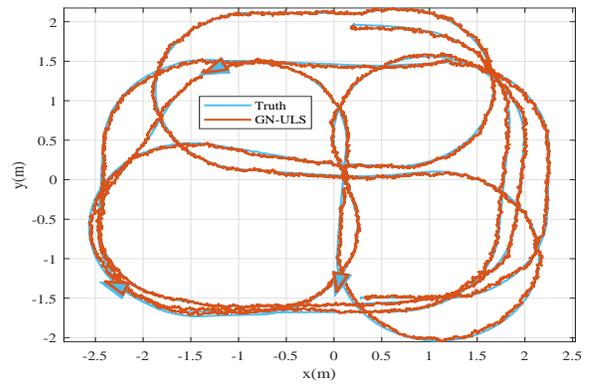}
		\caption{Dynamic experiment on the fast dataset}
		\label{fig::dynamic}
		\vspace{-15pt}
	\end{figure}
	\subsection{Discussions}
	All compared methods achieve similar accuracy in the experiments, given the standard deviation of noise in the magnitude of a centimeter. However, the \textbf{GN-ULS} and \textbf{GN-DAC} estimators significantly reduce the computation time. We attribute this advantage to the insights provided by Theorem~\ref{theorem: little op}. We also notice that, the \textbf{DAC} estimator is less robust to outliers than the \textbf{ULS} estimator. 
	
	The estimators' performance on the dynamic datasets is not as good as on the static datasets. The main reason is the complex indoor environment, where the distance-related bias $f(d^o)$ varies at different positions and orientations due to obstruction and signal reflection from the walls. 
	
	\section{Conclusion and Future Work}\label{section::conclusion}
	This work studies planar pose estimation using UWB range measurements. Grounded on a two-step scheme, we design an asymptotically efficient pose estimator \textbf{GN-ULS}. The proposed estimator defeats previous works in computational efficiency, stability, and accuracy under large noise. We also find that using a $\sqrt{M_T}-$consistent intermediate estimator, the divide-and-conquer estimation scheme followed by one step of Gauss-Newton iteration performs comparably well but less robust to outliers.
	
	In this work, we intend to present the possibility of using only range measurements for real-time pose estimation. The estimated trajectory is expected to be smoother and more accurate if integrated with motion models or odometry measurements. In future work, we plan to cope with the synchronization problem as the number of tags and sensors increases and the update frequency decreases, which is very important in large-scale dynamic scenarios.
		
	\useRomanappendicesfalse
	\appendices
	\section{Observability Analysis}\label{Appendix:: observability}
	Given that each tag communicates with each anchor, We prove that the necessary and sufficient condition for the planar pose to be observable in general cases is that there are at least three non-colinear anchors and at least two tags.\\
	
	\begin{proof}\\
	\textbf{Necessity.} Suppose the anchors are deployed on the same line $\mathcal{L}$, the global coordinate of each tag is not unique due to symmetry about line $\mathcal{L}$. As a result, the planar pose is not observable. Suppose we only deploy one tag, then the rotation is not observable.\\
	\textbf{Sufficiency.} Given three non-colinear anchors, the global coordinate of each tag is uniquely determined. Given the global and local coordinates of two tags on a rigid body, the planar pose is observable.
	\end{proof}	
	 
	\section{Proof of Theorem~\ref{theorem: little op}}\label{Appendix:: little op}
	We need the Helly-Bray Theorem~\cite{billingsley2013convergence} in the proof:
	\begin{lemma}
		Let $\{P_m\}$ be a sequence of probability measures on a sample space $\Omega$. Then $P_m$ converges weakly to $P$ if and only if
		\[
		\int_{\Omega}g(\omega)dP_m\rightarrow\int_{\Omega}g(\omega)dP
		\]
		for all bounded, continuous and real-valued functions on $\Omega$.
	\end{lemma}
	\begin{proof}
		Consider problem~\eqref{ULS} and we have $\sqrt{n}$-consistent estimates ${\bf R}(\hat\theta),\hat{\bf t}$. For simplicity, we omit the covariance matrix $\Sigma_{n}$ in the following proof. The optimal solution to~\eqref{ULS} is $\theta^*=\ln({\bf R}(\hat\theta)^{\top}\hat{{\bf R}}_{\rm ML})^{\vee}$ and ${\bf t}^*=\hat{{\bf t}}_{\rm ML}$, where $\ln(·)^{\vee}$ is the inverse map of~\eqref{eqn::angle paramerization}. $\theta^*$ and ${\bf t}^*$ should 
		satisfy the first order condition:  
		$$
		\frac{1}{n}{\bf J}(\theta^*,{\bf t}^*)^{\top}\left({\bf d}-{\bf f}(\theta^*,{\bf t}^*)\right)={\bf 0}
		$$
		Apply order-one Taylor expansion around $(0,{\bf \hat{t}})$ yields: 
		$$
		\frac{1}{n}{\bf J}_0^{\top}\left({\bf d}-{\bf f}(0,{\bf \hat t})\right)+{\bf G}\begin{bmatrix}
			\theta^*-0\\
			{\bf t}^*-{\bf \hat{t}}
		\end{bmatrix}+o\left(\begin{bmatrix}
			\theta^*-0\\
			{\bf t}^*-{\bf \hat{t}}
		\end{bmatrix}\right)={\bf 0}
		$$
		
		where
		$$
		{\bf G}=-\frac{1}{n}{\bf J}_0^{\top}{\bf J}_0+\frac{1}{n}\sum_{i=1}^{N}\sum_{m=1}^{M_{T}}\frac{\partial^2 \|f_{im}^{(0) }\|}{\partial (\theta,{\bf t})\partial (\theta,{\bf t})^{\top}}\big(d_{im}-\|f_{im}^{(0)}\|\big)
		$$
	Given that $\mathbf{R}(\hat\theta)$ and $\mathbf{t}$ are $\sqrt{n}-$consistent estimator by Theorem~\ref{theorem:: consistency after projection}, and Assumption~\ref{assumption::anchor distribution} holds, we can use Helly-Bray theorem and writes:
	\[
	\mathbf{G}=-\frac{1}{n}{\bf J}_0^{\top}{\bf J}_0+o_p(\bf{1})
	\]
	Notice that every row of ${\bf J}_0$, that is $\frac{\partial \|f_{im}^{(0)}\|}{\partial (\theta,{\bf t})^{\rm T}}$, is bounded, thus we have $\frac{1}{n}{\bf J}_0^{\top}{\bf J}_0$ bounded and:
	\[
	\mathbf{G}^{-1}=(-\frac{1}{n}{\bf J}_0^{\top}{\bf J}_0)^{-1}+o_p({\bf 1})
	\]
	Because $\mathbf{\hat{R}}_{ML}$ and $\mathbf{\hat{t}}_{ML}$ are $\sqrt{n}-$consistent by Theorem~\ref{theorem::ML}, we have:
	$$
	\begin{bmatrix}
		\theta^*-0\\
		{\bf t}^*-{\bf \hat{t}}
	\end{bmatrix}=\begin{bmatrix}
		\theta_{ML}-\theta^0+\theta^0-\hat{\theta}\\
		{\bf t}^*-{\bf \hat{t}}
	\end{bmatrix}=O_p(\frac{\bf 1}{\sqrt{n}})
	$$
	Thus
	\begin{align*}
		\frac{1}{n}{\bf J}_0^{\top}\left({\bf d}-{\bf f}(0,{\bf \hat t})\right)&=-\mathbf{G}\begin{bmatrix}
			\theta^*-0\\
			{\bf t}^*-{\bf \hat{t}}
		\end{bmatrix}+o_p(\frac{\bf 1}{\sqrt{n}})\\
		&=(-\frac{1}{n}{\bf J}_0^{\top}{\bf J}_0+o_p({\bf 1})) O_p(\frac{{\bf 1}}{\sqrt{n}})+o_p(\frac{\bf 1}{\sqrt{n}})\\
		&=O_p(\frac{\bf 1}{\sqrt{n}})
	\end{align*}

	\begin{align*}
		\begin{bmatrix}
			\theta^*\\
			{\bf t}^*
		\end{bmatrix}&=\begin{bmatrix}
			0\\
			{\bf \hat{t}}
		\end{bmatrix}-\mathbf{G}^{-1}\frac{1}{n}{\bf J}_0^{\top}\left({\bf d}-{\bf f}(0,{\bf \hat t})\right)+\mathbf{G}^{-1}o_p(\frac{\bf 1}{\sqrt{n}})\\
		&=\begin{bmatrix}
			0\\
			{\bf \hat{t}}
		\end{bmatrix}+(\frac{1}{n}{\bf J}_0^{\top}{\bf J}_0)^{-1}\frac{1}{n}{\bf J}_0^{\top}\left({\bf d}-{\bf f}(0,{\bf \hat t})\right)\nonumber\\&+\frac{1}{n}{\bf J}_0^{\top}\left({\bf d}-{\bf f}(0,{\bf \hat t})\right)o({\bf 1})+\mathbf{G}^{-1}o_p(\frac{\bf 1}{\sqrt{n}})\\
		&=\begin{bmatrix}
			\hat{\theta}_{\rm GN}\\
			{\bf \hat{t}}_{\rm GN}
		\end{bmatrix}+o_p(\frac{\bf 1}{\sqrt{n}})
	\end{align*}
	It follows then
	\begin{align*}
		{\bf \hat{t}}_{\rm GN}&={\bf \hat{t}}_{\rm ML}+o_p(\frac{\bf 1}{\sqrt{n}})\\
		{\bf\hat{R}_{\rm GN}}&={{\bf R}(\hat{\theta}+\theta^*+o_p(\frac{\bf 1}{\sqrt{n}})})={\bf{\hat R}}_{\rm ML}+o_p(\frac{\bf 1}{\sqrt{n}})
	\end{align*}
\end{proof}
\section{Proof of Lemma~1}\label{Appendix::lemma1}
Given three non-colinear anchors and two tags non-colinear with the origin of local reference frame, we prove that ${\bf H}=[{\bf H}_1\Gamma, \bf{H}_2]$ is full column rank.\\
\begin{proof}
	Write ${\bf H}$ as
	$$
	{\bf H}=[-2\bf{S}^{\top}\otimes\bar{\bf{A}}^{\top},-2\bf{1}_2\otimes\bar{\bf{A}}^{\top}]\begin{bmatrix}
		\Gamma & {\bf 0}\\
		{\bf 0} & {\bf I}_2
	\end{bmatrix}
	$$
	Here ${\bf\bar{A}}={\bf AP}$, and ${\bf P}$ is the projection matrix onto ${\rm Null}({\bf 1}_3)$. Multiplying ${\bf A}$ by ${\bf P}$ essentially subtract the average of anchor positions such that ${\bf\bar{A}1}_3={\bf 0}$. Under assumption~\ref{assumption::unique localizability}, we have ${\rm rank}({\bf\bar{A}})={\rm rank}({\bf\bar{S}})=2$. Using the property ${\rm rank}({\bf A}\otimes{\bf B})={\rm rank}({\bf A})\times{\rm rank}({\bf B})$, we have ${\rm rank}({\bf H}_1)={\rm rank}({\bf S}^{\top})\times {\rm rank}({\bf\bar{A}}^{\top})=4$, and ${\rm rank}([{\bf H_1},{\bf H_2}])={\rm rank}([{\bf S}^{\top},{\bf 1}_2])\times {\rm rank}({\bf\bar{A}}^{\top})=4$. Thus, ${\bf H}_2$ can be written as ${\bf H}_1{\bf C}$, where ${\bf C}\in\mathbb{R}^{4\times 2}$ is the column transformation matrix. We can then write:
	$$
	{\bf H}={\bf H}_1\left[{\bf I}_4,{\bf C}\right]\begin{bmatrix}
		\Gamma & {\bf 0}\\
		{\bf 0} & {\bf I}_2
	\end{bmatrix}
	$$
	Because ${\bf H}_1$ is full column rank, to prove ${\bf H}$ is full column rank, it sufficies to prove $\left[{\bf I}_4,{\bf C}\right]\begin{bmatrix}
		\Gamma & {\bf 0}\\
		{\bf 0} & {\bf I}_2
	\end{bmatrix}$ is full rank. Next, we derive the formula for ${\bf C}$.
	
	Because ${\bf S}^{\top}$ is full rank, we can write ${\bf 1}_2={\bf S}^{\top}\begin{bmatrix}
		\alpha\\ \beta
	\end{bmatrix}$ for some $\alpha$ and $\beta$, and $[\alpha,\beta]^\top$ is not a zero vector. Using the property $({\bf A}\otimes{\bf B})({\bf C}\otimes{\bf D})=({\bf AC}\otimes{\bf BD})$, we can write
	\begin{align*}
		{\bf H_2}&=-2\bf{1}_2\otimes\bar{\bf{A}}^{\top}=-2({\bf S}^{\top}\begin{bmatrix}
			\alpha\\ \beta
		\end{bmatrix})\otimes(\bar{\bf{A}}^{\top}{\bf I}_2)\\
		&=-2({\bf S}^{\top}\otimes\bar{\bf{A}}^{\top})(\begin{bmatrix}
			\alpha\\ \beta
		\end{bmatrix}\otimes{\bf I}_2)
	\end{align*}
	Thus we have 
	$$
	{\bf C}=(\begin{bmatrix}
		\alpha\\ \beta
	\end{bmatrix}\otimes{\bf I}_2)=\begin{bmatrix}
		\alpha & 0 & \beta & 0\\
		0      & \alpha & 0 &\beta\\
	\end{bmatrix}^{\top}$$
	We end the proof by noticing that 
	$$\left[{\bf I}_4,{\bf C}\right]\begin{bmatrix}
		\Gamma & {\bf 0}\\
		{\bf 0} & {\bf I}_2
	\end{bmatrix}=\begin{bmatrix}
		0  & 1 & \alpha & 0\\
		1  & 0 & 0      & \alpha\\
		-1 & 0 & \beta  & 0\\
		0  & 1 & 0      & \beta
	\end{bmatrix}$$ is full rank.
\end{proof}
\section{Proof of Theorem~\ref{theorem::consistency_of_LS_lin_estimate}}\label{Appendix:: big Op}
The proof is supported by the following lemma.
\begin{lemma} \label{lemma_root_mn_consistent}
	Let $\{X_k\}$ be a stationary sequence with $\mathbb E[X_k]=0$ and $\mathbb E\left[ X_k^2\right]\leq\infty$ for all k. It holds that $\sum_{k=1}^{n}X_k/\sqrt{n}=O_p(1)$.
	The proof of this lemma is straightforward using the Chebyshev's inequality.
\end{lemma}
\begin{proof}
	Write $\begin{bmatrix}
		\mathbf{\hat{y}} \\
		\mathbf{\hat{t}} \\
	\end{bmatrix}
	$ as 
	$(\frac{1}{n}{\bf H}^{\top}{\bf H})^{-1}\frac{1}{n}{\bf H}^{\top}{\bf{\bar{d}}}$, substitute ${\bf\bar{d}}=\bf{H}\begin{bmatrix}
		\mathbf{y}^{o} \\
		\mathbf{t}^{o} \\
	\end{bmatrix}+
	\begin{bmatrix}
		\bar{\bf{e_1}}\\
		\vdots\\
		\bar{\bf{e_N}}
	\end{bmatrix}$, we have:
	$$\begin{bmatrix}
		\mathbf{\hat{y}} \\
		\mathbf{\hat{t}} \\
	\end{bmatrix}
	=\begin{bmatrix}
		\mathbf{y}^{o} \\
		\mathbf{t}^{o} \\
	\end{bmatrix}+(\frac{1}{n}{\bf H^{\top}\bf{H})^{-1}}O_p(\mathbf{1}/\sqrt{n}).$$ The second term is from Lemma~\ref{lemma_root_mn_consistent}. $(\frac{1}{n}{\bf H^{\top}\bf{H})^{-1}}$ converges by Helly-Bray Theorem under Assumption~\ref{assumption::anchor distribution}, thus we can further write the second term as $O_p(\mathbf{1}/\sqrt{n})$.
\end{proof}

\section{Proof of Theorem~\ref{theorem:: consistency after projection}}\label{Appendix:: projection}
\begin{proof}
	For any $\hat{\mathbf R}$, we can verify that 
	$$
	\|\pi(\hat{\mathbf R})-\mathbf R^o\|_F\leq \|\pi(\hat{\mathbf R})
	-{\bf\hat R}\|_F+\|\hat{\mathbf R}-\mathbf R^o\|_F
	\leq 2\|\hat{\mathbf R}-\mathbf R^o\|_F
	$$
	Notice that $\|\hat{\mathbf{R}}-{\bf R}^o\|_{F}=\|\Gamma({\bf \hat{y}-y}^o)\|_2$, and $\hat{\mathbf y}-{\mathbf y}^o=O_p({\bf 1}/\sqrt{n})$, we conclude that each entry of $\mathbf{\pi(\hat{R})}-\mathbf{R}^o$ should be $O_p(1/\sqrt{n})$, i.e., $\pi(\hat{\mathbf R})$ is $\sqrt{n}$-consistent.
\end{proof}

\section{Divide-and-Conquer Estimator}\label{Appendix::DAC}
We estimate each tag's position using the Bias-Eli-Lin estimator proposed in~\cite{zeng2022global}, which is $\sqrt{M_T}-$consistent and with $O(M_T)$ computational complexity in this paper's notation. Let us denote these position estimates as $\tilde{{\mathbf{s}}_i}^{\mathcal{A}}$, and we have
\begin{equation}
	\tilde{{\mathbf{s}}_i}^{\mathcal{A}}={\mathbf{s}}_i^{\mathcal{A}}+O_p({\mathbf 1}/\sqrt{M_T})
\end{equation}

Next, we estimate the pose by minimizing the residual
$$ 
\sum_{i=1}^N\|\tilde{{\mathbf{s}}_i}^{\mathcal{A}}-{\mathbf Rs}_i^{\mathcal B}-t\|^2
$$

Use the same parameterization as~\eqref{GTRS} and discard the constraint $\|\mathbf y\|^2=1$, we formulate the following LS problem:
	\begin{equation}\label{LS-DAQ}
	\mathop{\rm min~}\limits_{\textbf{y,t}}
	\sum_{i=1}^N\|\tilde{{\mathbf{s}}_i}^{\mathcal{A}}-{\mathbf L}_i\Gamma{\mathbf y}-{\mathbf t}\|^2
\end{equation}
where ${\bf L}_i=({\bf s}_i^{\mathcal{B}}\otimes {\bf I}_2)^{\top}\in\mathbb{R}^{2\times 4}$ and $\Gamma=\begin{bmatrix}
	0 & 1 & -1 & 0\\
	1 & 0 & 0 & 1
\end{bmatrix}^{\top}$.

To prove that the closed-form solution of problem~\eqref{LS-DAQ} is $\sqrt{n}-$consistent, we can follow a similar procedure to the proof of Theorem~\ref{theorem::consistency_of_LS_lin_estimate}, and notice that 
$$
\frac{O_p({\mathbf 1}/\sqrt{M_T})}{\sqrt{N}}=O_p({\mathbf 1}/\sqrt{M_TN})=O_p({\mathbf 1}/n)
$$.

\section{Proof of Theorem~\ref{theorem::ML}}\label{Appendix:: theorem ML}
Let $\Theta=({\rm vec}({\bf R}), {\bf t})$ be the unknown parameter vector. First, we will show that the ML solution $\hat \Theta_{M_T}^{\rm ML}$ that optimally solves problem~\eqref{ML} is consistent. Before that, we define two functions on real sequences. 
\begin{definition}
	Let ${\bf p}\triangleq(p_i)_{i\in\mathbb N}$ and ${\bf q}\triangleq(q_i)_{i\in\mathbb N}$ be two sequences of real numbers, if $t^{-1} \sum_{i=1}^{t} p_i q_i$ converges to a real number, we call its limit, denoted as $\left\langle {\bf p},{\bf q} \right\rangle_t $, the tail product of ${\bf p}$ and ${\bf q}$.
	We call $\|{\bf p}\|_t\triangleq\sqrt{\langle p,p
		\rangle_t}$,
	if it exists, the tail norm of $p$.
\end{definition}

Define ${\bf d}^o(\Theta)\triangleq(d_j^o(\Theta))_{j=1}^{\infty}$, where $d_j^o(\Theta)=\|{\bf a}_{m}^{\mathcal{A}}-\bar {\bf L}_i \Theta\|/\sigma_{im}$, $\bar {\bf L}_i=\left[{\bf L}_i,{\bf I}_2\right]\in\mathbb{R}^{2\times 6}$ and $j=(m-1)N+i$. Note that $d_j^o(\Theta)$ is continuous with respect to ${\bf a}_{m}^{\mathcal{A}}$ and is bounded when ${\bf a}_{m}^{\mathcal{A}}$ is bounded. Then given Assumption~\ref{assumption::anchor distribution}, the tail norm $\|{\bf d}^o(\Theta)-{\bf d}^o(\Theta^o)\|_t^2$ exists by using the Helly-Bray theorem~\cite{billingsley2013convergence}.

Next, we will show that under Assumption~\ref{assumption::unique localizability} and~\ref{assumption::deployment_of_anchor}, the function $\|{\bf d}^o(\Theta)-{\bf d}^o(\Theta^o)\|_t^2$ has a unique minimum at $\Theta=\Theta^o$. By definition, we have $\|{\bf d}^o(\Theta)-{\bf d}^o(\Theta^o)\|_t^2$ equals
\begin{equation*}
	\frac{1}{N}\sum\limits_{i=1}^{N}\mathbb E_{{\bf a}^{\mathcal{A}} \sim \mu}\left[\left(\| {\bf a}^{\mathcal{A}}-\bar {\bf L}_i \Theta\|-\|{\bf a}^{\mathcal{A}}-\bar {\bf L}_i \Theta^o\| \right) ^2 \right], 
\end{equation*}
where $\mathbb E_{{\bf a}^{\mathcal{A}} \sim \mu}$ is taken over ${\bf a}^{\mathcal{A}}$ with respect to $\mu$. 

It is straightforward that when the tags are not colinear with the origin of local reference frame, for any $\Theta \neq \Theta^o$, there exists an ${\bf s}_i$ such that $\bar {\bf L}_i \Theta \neq \bar {\bf L}_i\Theta^o$. Suppose there exists a $\Theta \neq \Theta^o$ such that $\|{\bf d}^o(\Theta)-{\bf d}^o(\Theta^o)\|_t^2=0$. Then we have $\mu(\mathcal{A}_{\Theta{\bf s}_i})=1$, where $\mathcal{A}_{\Theta {\bf s}_i}=\{{\bf a}^{\mathcal{A}}\mid \| {\bf a}^{\mathcal{A}}-\bar {\bf L}_i \Theta\|=\|{\bf a}^{\mathcal{A}}-\bar {\bf L}_i\Theta^o\|\}$. Note that $\mathcal{A}_{\Theta {\bf s}_i}$ is the vertical bisector of the segment connecting $\bar {\bf L}_i \Theta$ and $\bar {\bf L}_i \Theta^o$, which contradicts the Assumption~\ref{assumption::deployment_of_anchor}. Hence, the function $\|{\bf d}^o(\Theta)-{\bf d}^o(\Theta^o)\|_t^2$ has a unique minimum at $\Theta=\Theta^o$.

Denote the objective function in~\eqref{ML} as $P_{M_T}({\Theta})$. We have 
\begin{align*}
	P_{M_T}({\Theta}) & = \frac{1}{n}\sum_{m=1}^{M_T}\sum_{i=1}^{N} \frac{(d^o_{im}({\Theta}^o)-d^o_{im}({\Theta})+r_{im})^2}{\sigma_{im}^2} \\
	& \rightarrow \|{\bf d}^o({\Theta}^o)-{\bf d}^o({\Theta})\|_t^2 + 2\langle {\bf d}^o({\Theta}^o)-{\bf d}^o({\Theta}),{\bf r}\rangle_t + \|{\bf r}\|_t^2 \\
	& = \underbrace{\|{\bf d}^o({\Theta}^o)-{\bf d}^o({\Theta})\|_t^2}_{P({\Theta})} + \|{\bf r}\|_t^2,
\end{align*}
where ${\bf r}\triangleq(r_{j}/\sigma_{j})_{j=1}^{\infty}$, $j=(m-1)N+i$, and $\langle {\bf d}^o({\Theta}^o)-{\bf d}^o({\Theta}),{\bf r}\rangle_t=0$ is based on~\cite[Theorem 3]{jennrich1969asymptotic}. Note that $\hat {\Theta}_{M_T}^{\rm ML}$ minimizes $P_{M_T}({\Theta})$ for any $M_T \in \mathbb N$. Then $(\hat {\Theta}_{M_T}^{\rm LS})$ forms a sequence of minimizers of $P_{M_T}({\Theta})$. Let ${\Theta}'$ be a limit point of the bounded sequence $(\hat {\Theta}_{M_T}^{\rm ML})$, and let $(\hat {\Theta}_{M_{T}^k}^{\rm ML})$ be any subsequence which converges to ${\Theta}'$. By the continuity of $P({\Theta})$ and the uniform convergence of $P_{M_T}({\Theta})$ to $P({\Theta})+\|{\bf r}\|_t^2$, $P_{M_T^k}(\hat {\Theta}_{M_{T}^k}^{\rm ML})  \rightarrow P({\Theta}')+ \|{\bf r}\|_t^2$ as $k \rightarrow \infty$. 
Since $\hat {\Theta}_{M_{T}^k}^{\rm ML}$ is the global minimizer of $P_{M_T^k}({\Theta})$, $P_{M_T^k}( \hat {\Theta}_{M_{T}^k}^{\rm ML})  \leq P_{M_T^k}({\Theta}^o)$. It follows that by letting $k \rightarrow \infty$, $P({\Theta}')+ \|{\bf r}\|_t^2 \leq P({\Theta}^o)+ \|{\bf r}\|_t^2=\|{\bf r}\|_t^2$. Hence $P({\Theta}')=0$. As we have proved, $P({\Theta})$ has a unique minimum at ${\Theta}^o$, which implies that ${\Theta}'={\Theta}^o$. Thus for almost every ${\bf r}$, $\hat {\Theta}_{M_T}^{\rm ML} \rightarrow {\Theta}^o$. 

Let $l({\bf d};{\Theta})$ be the log likelihood function, and denote the derivative of $l({\bf d};{\Theta})$ with respect to ${\Theta}$ as $l'({\bf d};{\Theta})$. Under the Gaussian noises Assumption~\ref{assumption::noise}, it holds that $l'({\bf d};{\Theta}^o) \xrightarrow[]{d} \mathcal N(0,{\bf F})$ where ${\bf F}$ is the information matrix~\cite{crowder1984constrained}. 
We can write the $SO(2)$ constraint as ${\bf f}({\Theta})=0$ as shown in~\eqref{SO_constraints} and its Jacobian as ${\bf dF}({\Theta}) \in \mathbb R^{3 \times 6}$.
Let ${\bf U}(\Theta)$ be the matrix whose columns form an orthonormal null space of ${\bf dF}({\Theta})$, i.e., ${\bf dF}({\Theta}) {\bf U}(\Theta)=0$ and ${\bf U}^\top (\Theta) {\bf U}(\Theta)={\bf I}$. Define ${\bf F}^*({\Theta})={\bf F}+M_T{\bf dF}^\top ({\Theta}){\bf dF} ({\Theta})$, then given the identifiability of the problem, ${\bf F}^*({\Theta})$ is nonsingular~\cite{crowder1984constrained}. Let 
\begin{align*}
	{\bf B}({\Theta}) & ={\bf F}^*({\Theta})^{-1} ({\bf F}+l''({\bf d};{\Theta})), \\
	{\bf Q}({\Theta}) & ={\bf F}^*({\Theta})^{-1} {\bf dF}({\Theta}) \left( {\bf dF}^\top({\Theta}) {\bf F}^*({\Theta})^{-1} {\bf dF}({\Theta}) \right)^{-1}.
\end{align*}
Since $\hat {\Theta}_{M_T}^{\rm ML}$ is consistent, we have ${\bf Q}(\hat {\Theta}_{M_T}^{\rm ML}) \xrightarrow[]{p} {\bf Q}({\Theta}^o)$ and 
\begin{align*}
	{\bf B}(\hat {\Theta}_{M_T}^{\rm ML}) & = \left( \frac{{\bf F}^*(\hat {\Theta}_{M_T}^{\rm ML})}{M_T}\right)^{-1} \left(\frac{{\bf F}+l''({\bf d};\hat {\Theta}_{M_T}^{\rm ML})}{M_T}\right)  \\
	& = \left( \frac{{\bf F}^*(\hat {\Theta}_{M_T}^{\rm ML})}{M_T}\right)^{-1} \left(\frac{ \mathbb E[-l''({\bf d};{\Theta}^o)]+l''({\bf d};\hat {\Theta}_{M_T}^{\rm ML})}{M_T}\right)  \\
	& \xrightarrow[]{p} 0.
\end{align*}
In addition, ${\bf Q}({\Theta}^o)$ is bounded. Then, based on~\cite[Theorem~3]{crowder1984constrained}, the covariance matrix of $\hat {\Theta}_{M_T}^{\rm ML}$ converges to $ \left( {\bf I}-{\bf Q}({\Theta}^o) {\bf dF}^\top ({\Theta}^o)\right){\bf F}^*({\Theta}^o)^{-1}$, which can be further transformed into~\cite{moore2008maximum} 
\begin{equation} \label{asymptotic_covariance}
	{\bf U}({\Theta}^o) \left( {\bf U}^\top({\Theta}^o) {\bf F} {\bf U}({\Theta}^o)\right)^{-1} {\bf U}^\top({\Theta}^o).
\end{equation}
Actually,~\eqref{asymptotic_covariance} is the constrained Cramer-Rao lower bound (CRLB), showing the ML solution $\hat{\Theta}_{M_T}^{\rm ML}$ that optimally solves problem~\eqref{ML} is asymptotically efficient. We will specifically discuss the CRLB and give the explicit expression of ${\bf F}$ and ${\bf U}({\Theta}^o)$ in Appendix~\ref{Appendix::CRLB}.

\section{$SO(2)$ Constrained Cramer-Rao Lower Bound}\label{Appendix::CRLB}
Suppose we want to estimate the unknown vector $\Theta^o=\left({\rm vec}({\bf R}^o\right), {\bf t}^o)\in\mathbb{R}^{6\times 1}$ from the range measurements $d_{im}$ corrupted by independent noise $r_{im}\sim\mathcal {N}(0,\sigma^{2}_{mi})$ for $m=1,2,\dots,M_{T}$, and $i=1,2,\dots,N$, where the observations follow model~\eqref{measurement model}. We can compute the CRLB for ${\Theta}^o$ as follows.

We shall first evaluate the  Fisher information matrix (FIM) of the unknown parameter vector ${\Theta}^o$ without having the ${\rm SO}(2)$ constraint. The covariance matrix of any unbiased estimate of the parameter vector ${\Theta}^o$ satisfies
$$\mathbb{E}\{ (\widehat{\Theta}-\Theta^o)(\widehat{\Theta}-\Theta^o)^{\top}\}\geq{\bf F}^{-1},$$
where ${\bf F}$ is the Fisher information matrix (FIM). 

Let us define $\bar{{\bf s}}_{i}=({\bf s}_i^{\mathcal{B}}, 1)\in\mathbb{R}^3$ for $i=1,2,\dots,N$. 
\begin{align}\label{eqn::fisher information}
	{\bf F}&=\sum_{m=1}^{M_T}\sum_{i=1}^{N}\mathbb{E}_{\Theta^o}\left\{\left(\frac{\partial \ln p(d_{im};\Theta^o)}{\partial \Theta}\right)\left(\frac{\partial \ln p(d_{im};\Theta^o)}{\partial \Theta}\right)^{\rm T}\right\}
	\\&=\sum_{m=1}^{M_T}\sum_{i=1}^{N}\frac{(\bar{{\bf s}}_{i}\otimes {\bf I}_2)({\bf a}_m-{\bf s}_i^{\mathcal{A}})({\bf a}_m-{\bf s}_i^{\mathcal{A}})^{\top}(\bar{{\bf s}}_{i}^{\top}\otimes {\bf I}_2)}{\sigma_{im}^{2}\|{\bf a}_m-{\bf s}_i^{\mathcal{A}}\|^2}.
\end{align}

The CRLB by imposing ${\bf R}$ to ${\rm SO}(2)$ is obtained by the FIM together with the gradient matrix of the constraints with respect to ${\Theta}$. Let ${\bf R}^o=\begin{bmatrix}
	\textbf{y}_1 & \textbf{y}_2 
\end{bmatrix}$, the constraint ${\bf R}\in{\rm SO}(2)$ can be expressed locally by $3$ continuously differentiable constraints~(with the constrained $\det(R)=1$ locally redundant):
\begin{equation} \label{SO_constraints}
	{\bf f}(\Theta^o)=\begin{bmatrix}
		{\bf y}_1^{\top}{\bf y}_1-1 \\
		{\bf y}_2^{\top}{\bf y}_1 \\
		{\bf y}_2^{\top}{\bf y}_2-1 \\     
	\end{bmatrix}=\textbf{0}_3\in\mathbb{R}^{3\times 1}.
\end{equation}
Let the gradient matrix of the constraints be defined by
$${\bf dF}(\Theta^o)=\frac{\partial{\bf f}(\Theta^o)}{\partial \Theta^{\top}}=\begin{bmatrix}
	2{\bf y}_1^{\top} & {\bf 0}_2^{\top} & {\bf 0}_2^{\top} \\
	{\bf y}_2^{\top} & {\bf y}_1^{\top} & {\bf 0}_2^{\top} \\
	{\bf 0}_2^{\top} & 2{\bf y}_2^{\top} & {\bf 0}_2^{\top}
\end{bmatrix}\in\mathbb{R}^{3\times 6}.$$
The gradient matrix ${\bf dF}(\Theta^o)$ has full row rank and there exists a matrix ${\bf U}$ whose columns form an orthonormal basis for the null space of ${\bf dF}(\Theta^o)$,

$$
{\bf U}=\frac{1}{\sqrt{2}}\begin{bNiceArray}{c|c}
	\bf{y_{2}}   & \Block{2-1}{\bf{O}_{4\times2}}\\
	\bf{-y_{1}}  &    \\\hline
	\bf{0}_2 & \sqrt{2}\bf{I_2}\\
\end{bNiceArray}\in\mathbb{R}^{6\times(6-3)}
$$

Given ${\bf U}^{\top}{\bf FU}$ nonsingular, then the covariance matrix of any unbiased estimate of $\Theta^o$ satisfies
$$\mathbb{E}\{(\widehat{\Theta}-\Theta^o)(\widehat{\Theta}-\Theta^o)^{\top}\}\geq{\bf U}({\bf U}^{\top}{\bf FU})^{-1}{\bf U}^{\top}\triangleq \textbf{CRLB}.$$

The theoretical lower bound $\textbf{CRLB}$ is related to the parameters in a complicated way. For example, a counter-intuitive fact is that we can influence the lower bound by simply moving the origin of the local frame $O^{\mathcal{B}}$ while moving $O^\mathcal{A}$ has no effect. This is because moving the origins does not change ${\bf a}_m-{\bf s}_i^{\mathcal{A}}$ but moving $O^{\mathcal{B}}$ does change $\bar{{\bf s}}_{i}$ in~\eqref{eqn::fisher information}.

It is still being determined whether or not we can devise sensor deployment strategies from the view of CRLB. In general, we should deploy the tags as sparsely as possible. This is because if we multiply ${\mathbf s}_i^{\mathcal{B}}$ by some scaling parameter $k>1$, ${\bf a}_m-{\bf s}_i^{\mathcal{A}}$ would barely change given the environment's scale significantly larger than the robot's, while $(\bar{{\bf s}}_{i}\otimes {\bf I}_2)$ would be multiplied by $k$. As a result, the FIM increases to $k^2$ times, and the CRLB decreases to $\frac{1}{k^2}$ times.

\section{Extension to different height setting}\label{Appendix::Extension}
If $\Delta h_{im}\neq 0$, the squared measurement model~\eqref{squared model} becomes
\begin{align}\label{modified squared model}
	d_{im}^2 &= \|\mathbf{a}_m^{\mathcal {A}}-\mathbf{Rs}_i^{\mathcal {B}}-\mathbf{t}\|^2+(\Delta h_{im})^2\\\nonumber&+2\sqrt{\|\textbf{a}_m^{\mathcal {A}}-\textbf{R}^o\textbf{s}_i^{\mathcal {B}}-\textbf{t}^o\|^2+(\Delta h_{im}^2)}r_{im}+r_{im}^2
\end{align}
Substract both sides by $(\Delta h_{im})^2+\sigma_{im}^2$, we have
\begin{equation}\label{modified s4}
	d_{im}^2-(\Delta h_{im})^2-\sigma_{im}^2=\|\textbf{a}_m^{\mathcal {A}}\|^2-2{\textbf{a}_m^{\mathcal {A}}}^{\top}(\textbf{Rs}_i^{\mathcal {B}}+\textbf{t})+\|\textbf{Rs}_i^{\mathcal {B}}+\textbf{t}\|^2+e'_{im}
\end{equation}
where $e'_{im}=2\sqrt{\|\textbf{a}_m^{\mathcal {A}}-\textbf{R}^o\textbf{s}_i^{\mathcal {B}}-\textbf{t}^o\|^2+(\Delta h_{im}^2)}r_{im}+r_{im}^2-\sigma_{im}^2$ has zero mean.

Notice the similarity between equations~\eqref{smodel.4} and~\eqref{modified s4}, we can follow almost the same derivation and analysis to obtain a $\sqrt{n}-$consistent estimator. As for the Gauss-Newton iteration, we transform problem~\eqref{ULS} as follows:
	\begin{equation}\label{modified ULS}
	\mathop{\rm min~}\limits_{\theta, {\bf t}} ~\sum_{i,m}\frac{  \left[d_{im}-\sqrt{\|{\bf a}_{m}^{\mathcal{A}}-{\bf L}_i{\rm vec}\left({\bf R}(\hat{\theta}+\theta)\right)-{\bf t}\|^2+(\Delta h_{im})^2}\right]^2}{\sigma_{im}^2}
\end{equation}
where ${\bf L}_i=({\bf s}_i^{\mathcal{B}}\otimes {\bf I}_2)^{\top}\in\mathbb{R}^{2\times 4}$ and $\mathbf{R}(\hat{\theta})=\mathbf{\pi(\hat{R})}$. 

Similarly, We use $\theta=0, \mathbf{t}=\hat{{\bf t}}$ as the initial value. Write $f_{im}(\theta, {\bf t})\triangleq{\bf a}_{m}^{\mathcal{A}}-{\bf L}_i{\rm vec}\left({\bf R}(\hat{\theta}+\theta)\right)-{\bf t}$, $f_{im}^{(0)}\triangleq f_{im}(0, \hat {\bf t})$, $g_{im}(\theta,t)\triangleq\sqrt{\|f_{im}(\theta, {\bf t})\|^2+(\Delta h_{im})^2}$, and $g_{im}^{(0)}\triangleq g_{im}(0,\hat{t})$. Derivatives of $g_{im}^{(0)}$, w.r.t $\theta$ and ${\bf t}$ are:
\begin{equation*}
	\frac{\partial g_{im}^{(0)}}{\partial (\theta,{\bf t})}=\begin{bmatrix}
		-\frac{1}{\|g_{im}^{(0)}\|}
		\Psi^{\top} ({\bf I}_2\otimes {\bf 
			R}(\hat\theta)^{\top})  {\bf L}_i^{\top}\,f_{im}^{(0)}\\
		-\frac{1}{\|g_{im}^{(0)}\|} f_{im}^{(0)}
	\end{bmatrix}
\end{equation*}
where $\Psi=\frac{\partial {\rm vec}({\bf R}(0))}{
	\partial \bf \theta}=\begin{bmatrix}
	0~1~-1~0
\end{bmatrix}^{\top}$.

Stacking the rows $\frac{g_{im}^{(0)}}{\partial (\theta,{\bf t})^{\top}}$ gives the modified ${\bf J}_0$. Stacking $g_{im}^{(0)}$ gives the vector ${\bf g}(0,{\bf\hat{t}})$. The one step iteration writes:
\begin{equation}\label{eqn::modified one step GN}
	(\hat{\theta},\hat{{\bf t}})_{\rm GN}=
	(0,\hat{\bf t})+
	({\bf J}_0^{\top}\Sigma_{n}^{-1}{\bf J}_0)^{-1}{\bf J}_0^{\top}\Sigma_{n}^{-1}\big({\bf d}-{\bf g}(0,\hat{{\bf t}})\big)
\end{equation}

The modification for proofs of Theorem~\ref{theorem: little op} and~\ref{theorem::ML} is straightforward but tedious, thus omitted here. 

The FIM~\eqref{eqn::fisher information} is slightly modified as:
\begin{equation}\label{eqn::modified fisher information}
	{\bf F}=\sum_{m=1}^{M_T}\sum_{i=1}^{N}\frac{(\bar{{\bf s}}_{i}\otimes {\bf I}_2)({\bf a}_m-{\bf s}_i^{\mathcal{A}})({\bf a}_m-{\bf s}_i^{\mathcal{A}})^{\top}(\bar{{\bf s}}_{i}^{\top}\otimes {\bf I}_2)}{\sigma_{im}^{2}\left[\|{\bf a}_m-{\bf s}_i^{\mathcal{A}}\|^2+(\Delta h_{im})^2\right]}.
\end{equation}

		
		
		
		
		
		

\section*{Acknowledgment}
This work was supported in part by the National Natural Science Foundation of China under grant no. 62273288, and in part by Shenzhen Science and Technology Program JCYJ20220818103000001. The authors would like to thank Prof Huihuan Qian, Dr. Kaiwen Xue, and Mr. Jiale Zhong from The Chinese University of Hong Kong, Shenzhen, for their help in the experiment conduction.
\bibliographystyle{IEEEtran}
\bibliography{sj_reference}
\end{document}